\documentclass[journal,twoside,web]{ieeecolor}
\usepackage{generic}
\usepackage{cite}
\usepackage{amsmath,amssymb,amsfonts}
\usepackage{graphicx}
\usepackage{algorithm,algorithmic}
\usepackage{url}  
\usepackage{hyperref}
\hypersetup{hidelinks=true}
\usepackage{textcomp}
\usepackage{cleveref}
\usepackage{booktabs}
\usepackage{subcaption}
\usepackage{stfloats}
\usepackage{comment}
\usepackage{float}
\usepackage{bm}
\usepackage{multirow}
\usepackage{caption}
\usepackage{amsthm}
\usepackage{etoolbox}
\AtEndEnvironment{proof}{\hfill$\blacksquare$}

\theoremstyle{plain}
\newtheorem{theorem}{Theorem}%
\newtheorem{lemma}[]{Lemma}
\newtheorem{proposition}[]{Proposition}
\newtheorem{corollary}{Corollary}

\theoremstyle{definition}

\newtheorem{assumption}[]{Assumption}
\newtheorem{example}[theorem]{Example}

\theoremstyle{remark}
\newtheorem{remark}{Remark}

\crefname{assumption}{Assumption}{Assumptions}
\Crefname{assumption}{Assumption}{Assumptions}
\crefname{theorem}{Theorem}{Theorems}
\Crefname{theorem}{Theorem}{Theorems}
\crefname{proposition}{Proposition}{Propositions}
\Crefname{proposition}{Proposition}{Propositions}
\crefname{corollary}{Corollary}{Corollaries}
\Crefname{corollary}{Corollary}{Corollaries}
\crefname{lemma}{Lemma}{Lemmas}
\Crefname{lemma}{Lemma}{Lemmas}
\crefname{remark}{Remark}{Remarks}
\Crefname{remark}{Remark}{Remarks}
\crefname{example}{Example}{Examples}
\Crefname{example}{Example}{Examples}
\crefname{section}{Section}{Sections}
\Crefname{section}{Section}{Sections}
\crefname{subsection}{Section}{Sections}
\Crefname{subsection}{Section}{Sections}
\crefname{subsubsection}{Section}{Sections}
\Crefname{subsubsection}{Section}{Sections}
\crefname{appendix}{Appendix}{Appendices}
\Crefname{appendix}{Appendix}{Appendices}
\crefname{subappendix}{Appendix}{Appendices}
\Crefname{subappendix}{Appendix}{Appendices}
\crefname{subsubappendix}{Appendix}{Appendices}
\Crefname{subsubappendix}{Appendix}{Appendices}
\crefname{figure}{Fig.}{Figs.}
\Crefname{figure}{Figure}{Figures}
\crefname{table}{Table}{Tables}
\Crefname{table}{Table}{Tables}
\crefname{equation}{}{}

\newcommand{\define}{\mathrel{\mathop:}=}

\makeatletter
\let\arxiv@includegraphics\includegraphics
\def\arxiv@logofile{\logoname.eps}
\renewcommand{\includegraphics}[2][]{\def\arxiv@arg{#2}\ifx\arxiv@arg\arxiv@logofile\else\arxiv@includegraphics[#1]{#2}\fi}
\makeatother

\def\BibTeX{{\rm B\kern-.05em{\sc i\kern-.025em b}\kern-.08em
    T\kern-.1667em\lower.7ex\hbox{E}\kern-.125emX}}
\begin{document}

\title{Scaling Law of Neural Koopman Operators}

\author{Abulikemu Abuduweili$^{*}$, Yuyang Pang$^{*}$, Feihan Li, and Changliu Liu%
\thanks{$^{*}$Equal contribution.}%
\thanks{Corresponding authors: A. Abuduweili (\mbox{abulikea@andrew.cmu.edu}) and C. Liu (\mbox{cliu6@andrew.cmu.edu}).}\\[2pt]
{\normalsize Robotics Institute, Carnegie Mellon University, Pittsburgh, PA, USA}\\[-\baselineskip]}

\IEEEaftertitletext{\vspace{-2\baselineskip}}

\maketitle

\begin{abstract}
Neural Koopman models approximate and control nonlinear dynamical systems through linear dynamics in a learned latent space. Their performance depends on the interplay between training sample size and latent dimension, for which systematic scaling laws remain poorly understood. This paper develops a theoretical and empirical framework for this problem. We derive an upper bound on the Koopman approximation error that separates a sampling error, from finite data, and a projection error, from a finite-dimensional representation. For a fixed observable space of dimension $n$ and $m$ independent samples, we establish a sufficient sample-size scale $m\gtrsim\kappa_n\log n$, where the leverage constant $\kappa_n\ge n$ characterizes the geometry of the observable space. For observable spaces with $\kappa_n=\mathcal{O}(n)$ and a projection error of order $n^{-\beta}$, balancing the bound yields the latent dimension $n^\star\asymp(m/\log m)^{1/(1+2\beta)}$. We also introduce two lightweight regularizers: a covariance loss that discourages redundant latent features and an inverse-control loss that encourages the latent model to preserve information about the applied inputs. Experiments on six mechanical and robotic systems reveal the interplay between data availability, latent dimension, and prediction accuracy. The regularizers improve control despite small changes in prediction accuracy, increasing mean tracking survival by $6\%$ on the G1 humanoid and $49\%$ on the Go2 quadruped at the largest dataset. Our results provide theoretical insight and practical guidance for scaling learned dynamical representations for prediction and control. Code and a project page are available at \url{https://intelligent-control-lab.github.io/Koopman-Scaling/}.
\end{abstract}


\section{Introduction}
\label{sec:introduction}

Controlling nonlinear dynamical systems remains a central challenge in robotics and control theory \cite{khalil2002nonlinear}.
Two dominant paradigms have emerged to address this challenge. Model-free control, including model-free Reinforcement Learning (RL), delivers remarkable task-level performance in complex domains like legged locomotion \cite{lee2020learning,he2024learning}, but is hampered by its high sample complexity and highly task-specific controllers that require retraining for new tasks or varying dynamics \cite{haarnoja2018soft, zhaoabsolute, fu2024humanplushumanoidshadowingimitation, luo2023perpetual, gmm2021song}.
In contrast, model-based approaches \cite{polydoros2017survey} decouple the dynamics modeling from control, allowing a general-purpose dynamics model to be reused across various downstream tasks.
Within this paradigm, direct methods like nonlinear model predictive control (MPC) tackle the full dynamics but require solving computationally expensive, nonconvex optimization at each time step \cite{mayne2000constrained, yi2024covompctheoreticalanalysissamplingbased, sotaro2023analytical}. Local linearization methods, such as iterative LQR, are more efficient but are only reliable near a nominal trajectory and can degrade unpredictably when the system deviates from this trajectory \cite{Kazemi_Majd_Moghaddam_2013, nagabandi2017neuralnetworkdynamicsmodelbased}. Addressing these deviations often requires sophisticated adaptation mechanisms to maintain model fidelity under uncertainty \cite{abuduweili2019adaptable,abuduweili2021robust}.
This trade-off motivates our focus: global linearization techniques, which represent complex nonlinear dynamics by linear approximations in a higher-dimensional `lifted' latent space. By paying a one-time offline modeling cost, these methods enable the use of efficient and powerful linear control techniques for real-time execution.

Originating in the 1930s \cite{koopman1931hamiltonian}, Koopman operator theory analyzes the evolution of observable (embedding) functions in the latent state, rather than tracking states in an original nonlinear space. It establishes that an infinite-dimensional linear operator governs this latent evolution, allowing full nonlinear systems to be analyzed with linear theory in a lifted function space \cite{mezic2005spectral}.
This framework is widely applied across fields like fluid dynamics \cite{rowley2009spectral}, power systems, and molecular dynamics.
Practical applications rely on data-driven methods that seek a finite-dimensional approximation of the Koopman operator via Extended Dynamic Mode Decomposition (EDMD) \cite{williams2015data}. EDMD requires a user-specified dictionary of basis functions (e.g., polynomials, radial basis functions) to define the latent space of observables, and then computes a finite-dimensional linear operator (matrix) via least-squares regression. Consequently, EDMD's success hinges entirely on dictionary selection, and its scalability is limited by the curse of dimensionality, as the number of basis functions often grows exponentially with the system's state dimension \cite{johnson2025heterogeneous}.
To overcome this limitation, recent research uses deep neural networks to \textit{learn} the observable functions directly from data. In this paradigm, often termed \textbf{Neural Koopman}, a neural network embedding function maps the original state into a latent space of observables. The entire model is trained end-to-end to simultaneously discover a set of basis functions (the embedding) and the linear operator governing the latent dynamics \cite{lusch2018deep,yeung2019learning}. This synergy between Koopman theory and deep learning has enabled significant progress in robotics, with applications ranging from motion planning \cite{kim2024learning, lyu2023task} to real-time, low-level control \cite{ shi2022deep, korda2018linear, mamakoukas2021derivative,chen2024korol}.
Parallel theoretical efforts seek to make these data-driven approximations rigorous. EDMD is known to converge to the Koopman operator as data and dictionary size grow \cite{korda2018convergence}. Recent works derive finite-data error bounds for prescribed dictionaries \cite{nuske2023finite,llamazares2026datadriven,fassler2026finite} and kernel-based learning rates \cite{kostic2022learning}, establish rigorous spectral approximation methods that avoid spurious eigenvalues \cite{colbrook2024rigorous}, propose spectral regularizers enforcing Lyapunov stability \cite{bakker2019learning, mamakoukas2021derivative}, and connect approximation errors to closed-loop guarantees \cite{strasser2026overview}. These analyses characterize a given observable space. In neural Koopman models, however, the observable space is itself a design choice.
On the other hand, integrating learned Koopman models with linear MPC has achieved high-fidelity tracking on challenging quadrupedal and humanoid systems \cite{li2025continuallearningliftingkoopman, yang2025koopmanlmpc} and whole-body safe control of robotic manipulators and legged robots \cite{Jung2026koppmansafe}.
Despite these empirical successes, the connection between Koopman approximation theory and the practical performance of finite-sample, finite-capacity models remains weak. This paper takes a step toward closing this gap.

However, the performance of these data-driven methods fundamentally depends on the interplay between the number of training samples $m$ and the latent dimension $n$ of the model. This observation is analogous to the study of \textbf{neural scaling laws} in deep learning, where model performance has been shown to scale predictably as a power-law of data, model size, and compute \cite{kaplan2020scaling, hoffmann2022training}.
In robotics, recent large-scale studies have empirically validated these scaling properties for imitation learning and offline reinforcement learning \cite{zitkovich2023rt, o2024open, kumar2023offline,sartor2024neural,kareeremergence}. Despite recent advances in Koopman approximation theory and neural dynamics learning, it remains unclear how training data and latent representation capacity should scale together. Existing approximation analyses primarily characterize prescribed observable spaces, whereas neural Koopman models introduce the additional problem of learning these spaces from data. Moreover, observable spaces of the same dimension may have substantially different statistical properties and sample requirements. Understanding this interaction is important not only for Koopman approximation but also for the broader problem of scaling learned representations for model-based decision making. This paper develops a theoretical and empirical framework connecting data availability, representation capacity, observable-space geometry, and downstream control performance. Our key contributions are:
\begin{itemize}
    \item \textbf{A theoretical framework establishing data--capacity scaling laws for Koopman approximation.} Building on leverage-based least-squares theory~\cite{cohen2013stability} and finite-data EDMD analyses, we derive approximation error bounds that separate sampling and projection errors, characterize the dimension dependence of the sampling error through a basis-invariant leverage constant, and extend to the linear-input predictor used in control. The analysis yields data--capacity scaling rules, including a dimension-selection rule under a fixed data budget, and strong consistency under joint data--capacity scaling.
    \item \textbf{Two lightweight regularizers for neural Koopman learning.} Motivated by feature conditioning and the representation of physical actuation, we introduce covariance and inverse-control losses that complement conventional prediction objectives and improve downstream control performance.
    \item \textbf{Systematic theoretical and empirical validation.} Through comprehensive experiments that test the individual error components and end-to-end neural Koopman learning across six mechanical and robotic systems, we investigate the interplay between data availability, latent dimension, and representation geometry, and evaluate its implications for prediction and model predictive control (MPC). Code and a project page are available at \url{https://intelligent-control-lab.github.io/Koopman-Scaling/}.
\end{itemize}

\begin{figure*}[htbp]
    \centering
    \vspace{-5pt}
    \begin{minipage}[b]{0.3\linewidth}
        \centering
        \includegraphics[width=\linewidth]{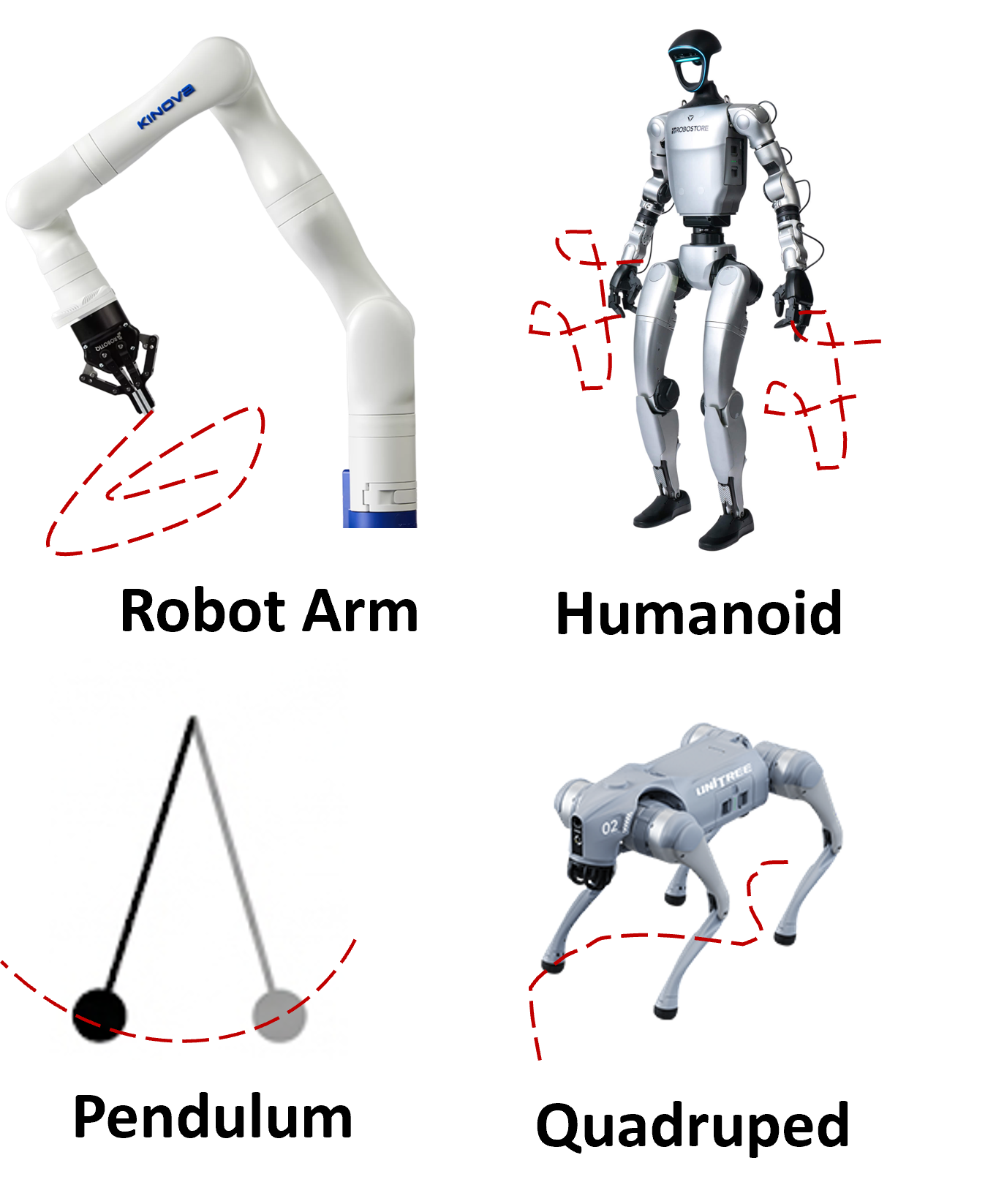}
        \caption*{(a) Nonlinear Robotic Systems  \label{fig:nonlinear_robots}}
    \end{minipage}
    \hfill
    \begin{minipage}[b]{0.65\linewidth}
        \centering
        \includegraphics[width=\linewidth]{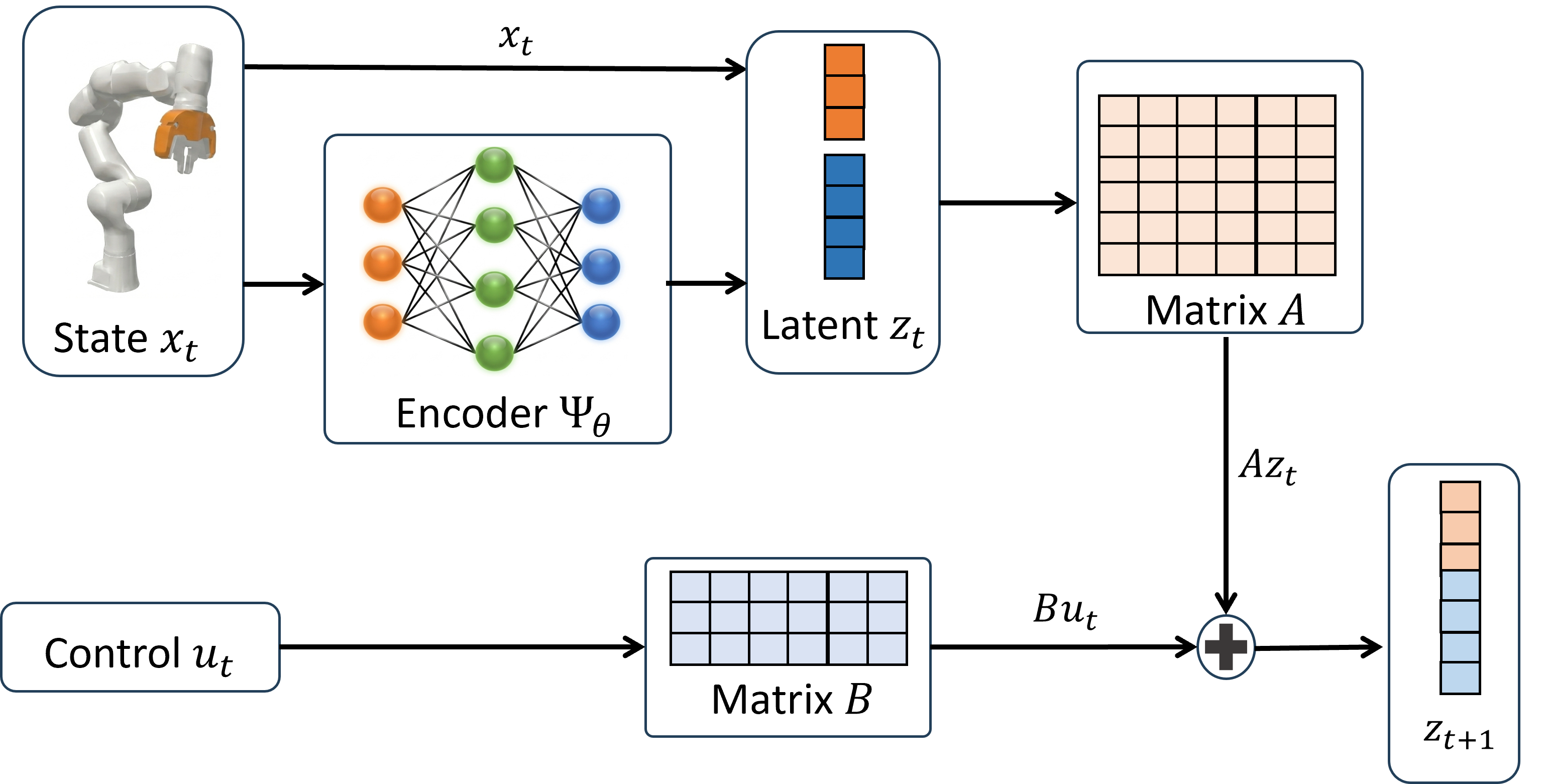}
        \caption*{(b) Koopman Neural Network Architecture         \label{fig:koopman_nn}}
    \end{minipage}
    \vspace{-5pt}
    \caption{Overview.
    (a) Mechanical systems (e.g., robot arm, quadruped) exhibiting complex nonlinear dynamics in physical state space.
    (b) Our proposed Koopman architecture uses an encoder network $\Psi_\theta$ whose features, together with the physical state $x_t$, form the latent state $z_t$, governed by linear dynamics: $z_{t+1} = \mathbf{A}z_t + \mathbf{B}u_t$. }
    \label{fig:koopman_system}
    \vspace{-5pt}
\end{figure*}

\section{Preliminary: Koopman Operator Theory}
\label{sec:koopman_operator_theory}

We briefly overview the Koopman operator framework. Consider a discrete-time, nonlinear autonomous system (e.g., the robotic platforms in \cref{fig:koopman_system}a) governed by: $s^+ = f(s)$ for $s \in \mathcal{S}$,
where $s \in \mathcal{S}$ is the state and $s^+$ is the next state. Koopman theory analyzes this system by embedding the state into a latent space with linear dynamics via a vector-valued (observable) \textbf{embedding function}, $\boldsymbol{\Phi}^\star:\mathcal{S}\to\mathcal{O}$. The infinite-dimensional Koopman operator $\mathcal{K}$ propagates the latent state \textbf{linearly} forward in time:
\begin{align}
    (\mathcal{K}\boldsymbol{\Phi}^\star)(s) = \boldsymbol{\Phi}^\star(f(s)).
\end{align}

Now consider a controlled nonlinear system $x^+ = f(x, u)$, where $x \in \mathcal{X} \subseteq \mathbb{R}^{n_x}$ is the physical state and $u \in \mathcal{U} \subseteq \mathbb{R}^{n_u}$ is the control input. We seek a finite-dimensional observable map $\Phi:\mathcal{X}\to\mathbb{R}^{n}$ such that the lifted dynamics are approximated by the linear predictor
\begin{align}
    \Phi(x^+) \approx \mathbf{A} \Phi(x) + \mathbf{B} u, \label{eq:koopman_lin_dyn}
\end{align}
where $\mathbf{A} \in \mathbb{R}^{n \times n}$ describes the state-dependent evolution in the lifted space and $\mathbf{B} \in \mathbb{R}^{n \times n_u}$ captures the effect of the externally supplied control inputs.

To avoid degenerate trivial solutions (e.g., $\mathbf{A}=\mathbf{B}=0$ with $\Phi\equiv 0$) and preserve the physical interpretation of the state, we include the physical state in the embedding, as shown in \cref{fig:koopman_system}(b):
\begin{align}
    z := \Phi_\theta(x) = \begin{bmatrix} x \\ \Psi_\theta(x) \end{bmatrix}, \qquad \Psi_\theta:\mathbb{R}^{n_x}\to\mathbb{R}^{\,n-n_x}
\end{align}
where the full embedding $\Phi_\theta$ concatenates $x$ with residual nonlinear features learned by the neural encoder $\Psi_\theta$. Since the latent vector explicitly contains the physical state, $x$ is perfectly recovered via a fixed linear projection $\mathbf{P}$:
\begin{align}
    x = \mathbf{P} z = \mathbf{P} \Phi_\theta(x),
    ~ \mathbf{P} = \big[ \mathbf{I}_{n_x} , \mathbf{0}_{n_x\times (n-n_x)} \big] \in \mathbb{R}^{n_x\times n}
\end{align}
This construction enables state-constrained formulations to be expressed directly in the latent space while retaining linear system evolution. Both the embedding function $\Phi_\theta$ and the Koopman matrices $\mathbf{A}$ and $\mathbf{B}$ are learned end-to-end. The trainable parameters are the encoder weights $\theta$ and the matrices $\mathbf{A}$ and $\mathbf{B}$.

\section{Koopman Approximation and Scaling Analysis}
\label{sec:sup_proof}

We decompose the Koopman approximation error into two parts. The \emph{sampling error} comes from finite data. The \emph{projection error} comes from a finite-dimensional observable space. We first bound both terms for a fixed dictionary $\bm{\Phi}$, as in EDMD, and make their dependence on the latent dimension $n$ explicit. We then use the bounds to derive how data and latent dimension should be scaled together.
The neural models of \cref{sec:method} learn $\bm{\Phi}_\theta$ and $(\mathbf{A},\mathbf{B})$ jointly by gradient descent on a multi-step loss. The fixed-dictionary results apply to them conditionally on the learned features, for example when the transition matrix is fitted on transitions that were not used to train the encoder.
We present the analysis for autonomous dynamics $s^+=f(s)$. \Cref{rem:controlled} extends the sampling bound to the least-squares fit of the linear-input predictor in \cref{eq:koopman_lin_dyn}, with inputs drawn i.i.d.\ with the states.

Existing Koopman approximation theory establishes asymptotic convergence~\cite{korda2018convergence} and finite-data error bounds for prescribed observable spaces. N\"uske et al.~\cite{nuske2023finite} and later work~\cite{llamazares2026datadriven,fassler2026finite} bound the estimation error of EDMD for a given dictionary, and Kostic et al.~\cite{kostic2022learning} give learning rates for kernel estimators; their use in control is surveyed in~\cite{strasser2026overview}. Building on these results and on leverage-based least-squares theory~\cite{cohen2013stability}, we investigate how sample size and latent dimension should scale together. We measure the error on the observables of interest, make the dimension dependence of the sampling error explicit through the leverage constant, balance it against an approximation-rate assumption to obtain a dimension-selection rule, and examine how this trade-off appears in learned neural representations and in downstream prediction and control.

\subsection{Setting and Error Decomposition}
\label{sec:preliminary}

\subsubsection{Koopman operator}
Let $\mu$ be the distribution from which training states are sampled. Let $\mathcal{H}=L_2(\mathcal{S},\mu)$ be the Hilbert space of observables, with inner product $\langle g_1, g_2 \rangle_{\mathcal{H}} \define \int_{\mathcal{S}} g_1(s)\, g_2^*(s)\, d\mu(s)$. The Koopman operator acts on $\psi\in\mathcal{H}$ as $(\mathcal{K}\psi)(s) = \psi(f(s))$.
\begin{assumption}[Bounded Koopman Operator]\label{assumption:koopman_bounded}
$\mathcal{K}:\mathcal{H}\to\mathcal{H}$ is bounded, with $\|\mathcal{K}\|_{\mathcal{L}(\mathcal{H})}\le M<\infty$.
\end{assumption}
Let $\mu^+$ denote the distribution of the next state $s^+=f(s)$ when the current state $s$ is sampled from $\mu$. \Cref{assumption:koopman_bounded} holds if there is a constant $M$ such that $\mu^+(\mathcal{A})\le M^2\mu(\mathcal{A})$ for every measurable set $\mathcal{A}\subseteq\mathcal{S}$. Under this condition,
\begin{align*}
\|\mathcal{K}\psi\|_{\mathcal{H}}^2=\mathbb{E}_{\mu^+}\big[|\psi|^2\big]\le M^2\,\mathbb{E}_{\mu}\big[|\psi|^2\big]=M^2\|\psi\|_{\mathcal{H}}^2.
\end{align*}
Intuitively, the condition limits how much the dynamics can concentrate probability mass relative to the training-state distribution. If $\mu$ is invariant under the dynamics, then $\mu^+=\mu$ and $M=1$. Our analysis does not require invariance.

\subsubsection{EDMD}
Given a dictionary $\bm{\Phi}(s) = [\phi_1(s), \dots, \phi_n(s)]^\top \in \mathbb{R}^n$, let $\mathcal{H}_n = \text{span}\{\phi_1, \dots, \phi_n\}$ and let $\mathcal{P}_n$ be the $L_2(\mu)$-orthogonal projection onto $\mathcal{H}_n$. For $m$ data pairs $\{(s_i, s_i^{+})\}_{i=1}^{m}$, EDMD solves
\begin{align}
    \min_{\mathbf{K} \in \mathbb{R}^{n \times n}} \sum_{i=1}^m \left\| \bm{\Phi}(s_i^+) - \mathbf{K} \bm{\Phi}(s_i) \right\|_2^2. \label{eq:k_edmd}
\end{align}
If the empirical Gram matrix $\mathbf{G}_{n,m}$ is invertible, the solution is
\begin{align}
    \mathbf{K}_{n,m} &= \mathbf{A}_{n,m} \mathbf{G}_{n,m}^{-1}, \label{eq:koopman_k_solution} \\
 \mathbf{G}_{n,m} &= \frac{1}{m} \sum_{i=1}^m \bm{\Phi}(s_i) \bm{\Phi}(s_i)^\top, \nonumber\\
 \mathbf{A}_{n,m} &= \frac{1}{m} \sum_{i=1}^m \bm{\Phi}(s_i^+) \bm{\Phi}(s_i)^\top.   \nonumber
\end{align}
The population counterparts are $\mathbf{G}_n=\mathbb{E}_{s\sim\mu}[\bm{\Phi}(s)\bm{\Phi}(s)^\top]$, $\mathbf{A}_n=\mathbb{E}_{s\sim\mu}[\bm{\Phi}(f(s))\bm{\Phi}(s)^\top]$, and $\mathbf{K}_n=\mathbf{A}_n\mathbf{G}_n^{-1}$.

The matrix $\mathbf{K}$ propagates feature vectors, $\bm{\Phi}(s^+)\approx\mathbf{K}\bm{\Phi}(s)$. A scalar observable in $\mathcal{H}_n$ is a linear combination $\psi=\mathbf{c}^\top\bm{\Phi}$ with coefficient vector $\mathbf{c}\in\mathbb{R}^n$. Applying the Koopman operator to it gives
\begin{align*}
(\mathcal{K}\psi)(s)=\psi(f(s))=\mathbf{c}^\top\bm{\Phi}(f(s))&\approx\mathbf{c}^\top\mathbf{K}\bm{\Phi}(s)\\
&=(\mathbf{K}^\top\mathbf{c})^\top\bm{\Phi}(s).
\end{align*}
Thus $\mathbf{K}$ acts on feature vectors, and its transpose acts on coefficient vectors: the next-step value of $\psi$ has coefficients $\mathbf{K}^\top\mathbf{c}$. We write $\mathcal{K}_{n,m}$ and $\mathcal{K}_n$ for the operators on $\mathcal{H}_n$ with coefficient maps $\mathbf{c}\mapsto\mathbf{K}_{n,m}^\top\mathbf{c}$ and $\mathbf{c}\mapsto\mathbf{K}_{n}^\top\mathbf{c}$.

The population EDMD operator is the Galerkin projection of $\mathcal{K}$ onto $\mathcal{H}_n$~\cite{korda2018convergence}, $\mathcal{K}_n=\mathcal{P}_n\mathcal{K}|_{\mathcal{H}_n}$. To see this, let $\psi=\mathbf{c}^\top\bm{\Phi}$ and write the projection as $\mathcal{P}_n\mathcal{K}\psi=\mathbf{d}^\top\bm{\Phi}$. The coefficients $\mathbf{d}$ minimize $\|\mathcal{K}\psi-\mathbf{d}^\top\bm{\Phi}\|_{\mathcal{H}}^2$, so they satisfy the normal equations
\begin{align*}
\mathbf{G}_n\mathbf{d}=\mathbb{E}_\mu\big[\bm{\Phi}(s)\,\psi(f(s))\big]=\mathbb{E}_\mu\big[\bm{\Phi}(s)\bm{\Phi}(f(s))^\top\big]\mathbf{c}=\mathbf{A}_n^\top\mathbf{c}.
\end{align*}
Hence $\mathbf{d}=\mathbf{G}_n^{-1}\mathbf{A}_n^\top\mathbf{c}=\mathbf{K}_n^\top\mathbf{c}$, which is the coefficient map of $\mathcal{K}_n$. The empirical operator $\mathcal{K}_{n,m}$ is the same projection taken with respect to the empirical distribution of the samples.

\subsubsection{Koopman approximation error}
Our objective is to characterize the error between the empirical operator $\mathcal{K}_{n,m}$ and the true operator $\mathcal{K}$. Since $\mathcal{K}_{n,m}$ acts on the finite-dimensional space $\mathcal{H}_n$, we extend it to $\mathcal{H}$ through the orthogonal projection $\mathcal{P}_n$. The resulting operator $\mathcal{K}_{n,m}\mathcal{P}_n$, however, cannot approximate all observables of an infinite-dimensional space uniformly. To see this, suppose that $\mu$ is invariant, so that $\|\mathcal{K}\psi\|_{\mathcal{H}}=\|\psi\|_{\mathcal{H}}$, and take a unit observable $\psi$ orthogonal to $\mathcal{H}_n$. Then $\mathcal{P}_n\psi=0$, and
\begin{align*}
\|(\mathcal{K}_{n,m}\mathcal{P}_n-\mathcal{K})\psi\|_{\mathcal{H}}=\|\mathcal{K}\psi\|_{\mathcal{H}}=\|\psi\|_{\mathcal{H}}=1
\end{align*}
for every $n$ and $m$, so the operator-norm error $\|\mathcal{K}_{n,m}\mathcal{P}_n-\mathcal{K}\|_{\mathcal{L}(\mathcal{H})}$ does not vanish in general. For prediction and control, however, we are interested in a particular class of observables rather than the entire space. We therefore measure the error over an observable class $\mathcal{V}\subseteq\mathcal{H}$:
\begin{align}
\epsilon(n, m)\define \sup_{\psi \in \mathcal{V},\, \|\psi\|_{\mathcal{H}}\le 1} \big\| (\mathcal{K}_{n,m} \mathcal{P}_n - \mathcal{K})\psi \big\|_{\mathcal{H}}.
\label{eq:total_error_initial}
\end{align}
The most important case is $\mathcal{V}=\text{span}\{x_1,\dots,x_{n_x}\}$, the coordinates of the physical state. Our embedding retains the state \cref{eq:koopman_arch}, so each $x_j$ lies in $\mathcal{H}_n$ and $\mathcal{P}_nx_j=x_j$. The function $\mathcal{K}_{n,m}x_j$ is the predicted $j$-th coordinate $\hat{x}^+_j$ of the next state, and $\mathcal{K}x_j=x_j\circ f$ is its true value $x^+_j$. Summing over the coordinates gives the one-step prediction error of the physical state:
\begin{align}
\mathbb{E}_{\mu}\|\hat{x}^+-x^+\|_2^2&=\sum_{j=1}^{n_x}\|(\mathcal{K}_{n,m}-\mathcal{K})x_j\|_{\mathcal{H}}^2\nonumber\\
&\le\epsilon(n,m)^2\,\mathbb{E}_{\mu}\|x\|_2^2.\label{eq:one_step_mse}
\end{align}

To bound $\epsilon(n,m)$, we use the population operator $\mathcal{K}_n$ as an intermediate reference:
\begin{align*}
\mathcal{K}_{n,m}\mathcal{P}_n-\mathcal{K}=(\mathcal{K}_{n,m}-\mathcal{K}_n)\mathcal{P}_n+(\mathcal{K}_n\mathcal{P}_n-\mathcal{K}).
\end{align*}
The triangle inequality then gives
\begin{align}
\epsilon(n,m)\le\epsilon_{\text{samp}}(n,m)+\epsilon_{\text{proj}}(n), \label{eq:err_decompose}
\end{align}
with the sampling and projection errors
\begin{align*}
\epsilon_{\text{samp}}(n,m)&\define\|(\mathcal{K}_{n,m}-\mathcal{K}_{n})\mathcal{P}_{n}\|_{\mathcal{L}(\mathcal{H})},\\
\epsilon_{\text{proj}}(n)&\define\sup_{\psi \in \mathcal{V},\, \|\psi\|_{\mathcal{H}}\le 1}\|(\mathcal{K}_n\mathcal{P}_{n}-\mathcal{K})\psi\|_{\mathcal{H}}.
\end{align*}
The \textbf{sampling error} $\epsilon_{\text{samp}}$ measures the error of estimating the population operator from finite data within a fixed observable space. Because $\mathcal{H}_n$ is finite-dimensional, it is the operator norm of a finite-rank operator and can be bounded uniformly. The \textbf{projection error} $\epsilon_{\text{proj}}$ measures the approximation error inherent to the finite-dimensional representation, even with unlimited data. The following sections establish how these two components depend on the sample size $m$ and the latent dimension $n$.

\subsection{Sampling Error: How Much Data Does a Latent Dimension Need?}
\begin{assumption}[i.i.d. Samples]\label{assumption:iid}
    The pairs $(s_i, s_i^+=f(s_i))$, $i=1,\dots,m$, are drawn independently with $s_i\sim\mu$. The dictionary $\bm{\Phi}$ is fixed independently of these samples, e.g., learned on a separate data split.
\end{assumption}
\begin{assumption}[Bounded Features]\label{assumption:bounded}
    For each latent dimension $n$, there exists a finite constant $B_n$ such that $\|\bm{\Phi}(s)\|_2 \le B_n$ for all $s \in \mathcal{S}$.
\end{assumption}
\begin{assumption}[Linearly Independent Features]\label{assumption:gram}
    The population Gram matrix satisfies $\lambda_{\min}(\mathbf{G}_{n}) \ge \gamma_n > 0$.
\end{assumption}
As the latent dimension increases, the magnitude and conditioning of the features may change, so both $B_n$ and $\gamma_n$ depend on $n$. These quantities also depend on the particular basis used to represent the observable space. To characterize the statistical sample requirement independently of the feature parameterization, we introduce the \textbf{leverage constant} of $\mathcal{H}_n$:
\begin{align}
\kappa_n \define \sup_{s\in\mathcal{S}}\ \bm{\Phi}(s)^\top\mathbf{G}_n^{-1}\bm{\Phi}(s).\label{eq:kappa_def}
\end{align}
To interpret it, consider the normalized features $\tilde{\bm{\Phi}}=\mathbf{G}_n^{-1/2}\bm{\Phi}$. They are orthonormal in $\mathcal{H}$, because
\begin{align*}
\mathbb{E}_\mu\big[\tilde{\bm{\Phi}}\tilde{\bm{\Phi}}^\top\big]=\mathbf{G}_n^{-1/2}\mathbf{G}_n\mathbf{G}_n^{-1/2}=\mathbf{I}_n,
\end{align*}
and $\bm{\Phi}(s)^\top\mathbf{G}_n^{-1}\bm{\Phi}(s)=\|\tilde{\bm{\Phi}}(s)\|_2^2$. The leverage constant is therefore the largest squared norm of the normalized feature vector over the state space, and its average over $\mu$ is
\begin{align*}
\mathbb{E}_\mu\|\tilde{\bm{\Phi}}(s)\|_2^2=\mathbb{E}_\mu\big[\operatorname{tr}(\mathbf{G}_n^{-1}\bm{\Phi}\bm{\Phi}^\top)\big]=\operatorname{tr}(\mathbf{G}_n^{-1}\mathbf{G}_n)=n.
\end{align*}
A supremum is at least the mean, so $\kappa_n\ge n$. For an upper bound, \cref{assumption:bounded,assumption:gram} give
\begin{align*}
\|\tilde{\bm{\Phi}}(s)\|_2^2&\le\lambda_{\max}(\mathbf{G}_n^{-1})\,\|\bm{\Phi}(s)\|_2^2\\
&=\|\bm{\Phi}(s)\|_2^2/\lambda_{\min}(\mathbf{G}_n)\le B_n^2/\gamma_n.
\end{align*}
Together,
\begin{align}
n\ \le\ \kappa_n\ \le\ B_n^2/\gamma_n. \label{eq:kappa_range}
\end{align}
In particular, $B_n^2/\gamma_n$ must grow at least linearly in $n$.

The leverage constant depends only on the subspace $\mathcal{H}_n$, not on the basis that spans it. For an invertible matrix $\mathbf{T}$, the basis $\mathbf{T}\bm{\Phi}$ spans the same subspace and has Gram matrix $\mathbf{T}\mathbf{G}_n\mathbf{T}^\top$, so
\begin{align*}
(\mathbf{T}\bm{\Phi})^\top(\mathbf{T}\mathbf{G}_n\mathbf{T}^\top)^{-1}(\mathbf{T}\bm{\Phi})=\bm{\Phi}^\top\mathbf{G}_n^{-1}\bm{\Phi}.
\end{align*}
Hence, for any $\mu$-orthonormal basis $\{e_i\}$ of $\mathcal{H}_n$, $\kappa_n=\sup_s\sum_{i=1}^n|e_i(s)|^2$ (the supremum of the inverse Christoffel function of $\mathcal{H}_n$). Since the average squared magnitude of the normalized features is $n$, the ratio $\kappa_n/n$ measures how large their maximum magnitude is relative to this average. A large ratio indicates that some states have a disproportionately large influence on least-squares estimation. Consequently, observable spaces of the same dimension can have different sufficient sample requirements, even though their regression problems have the same number of parameters. The condition $m\gtrsim\kappa_n\log(n/\delta)$ is the stability condition of Cohen et al.~\cite{cohen2013stability} for least squares. Weighted least squares, with samples drawn from a density adapted to $\mathcal{H}_n$ and matching weights, satisfies the same condition with $\kappa_n$ replaced by $n$~\cite{cohen2017optimal}. We extend this condition to the cross-moment matrix and to the Koopman sampling error.

\paragraph*{Why the logarithm is needed}
The factor $\log n$ cannot be removed in general. Partition $\mathcal{S}$ into $n$ cells $\mathcal{C}_1,\dots,\mathcal{C}_n$ of $\mu$-mass $1/n$ each, and let $\mathcal{H}_n$ be spanned by their indicator functions. The functions $e_i=\sqrt n\,\mathbb{1}_{\mathcal{C}_i}$ form an orthonormal basis. Every state lies in exactly one cell, so $\sum_ie_i(s)^2=n$ and $\kappa_n=n$. The empirical Gram matrix $\mathbf{G}_{n,m}$ is diagonal, with entries proportional to the number of samples in each cell, so it is invertible only if every cell contains a sample. This is the coupon-collector problem: for any fixed $\varepsilon>0$, if $m\le(1-\varepsilon)\,n\log n$, the probability that every cell is sampled tends to zero as $n\to\infty$~\cite{erdos1961classical}. No condition of the form $m\gtrsim\kappa_n$ without the logarithm can therefore guarantee invertibility.

\begin{proposition}[Sampling Error] \label{prop:samp_err}
Under \cref{assumption:koopman_bounded,assumption:iid,assumption:bounded,assumption:gram}, let $\delta\in(0,1)$, $\ell=\log(2n/\delta)$, and $\bar M=\max\{1,M\}$. If
\begin{align}
m\ \ge\ 12\,\kappa_n\,\ell,\label{eq:stability_threshold}
\end{align}
then, with probability at least $1-2\delta$, the empirical Gram matrix $\mathbf{G}_{n,m}$ is invertible and
\begin{align}
    \epsilon_{\text{samp}}(n, m) \ \le\ 4\bar M\left( \sqrt{\frac{2\kappa_n\ell}{m}} + \frac{\kappa_n\ell}{m}\right). \label{eq:thoery_est_err3}
\end{align}
\end{proposition}

\begin{proof}
The proof has four steps: we reduce the operator error to a matrix error in an orthonormal basis, split it into a Gram-matrix term and a cross-moment term, bound each term with a matrix concentration inequality, and combine the bounds.

\textit{Step 1 (reduction to an orthonormal basis).}
To express the operator error through spectral norms of matrices, we transform the features into an orthonormal basis, in which the norm of an observable equals the Euclidean norm of its coefficient vector. EDMD does not depend on the choice of basis of $\mathcal{H}_n$. Let $\mathbf{T}\in\mathbb{R}^{n\times n}$ be invertible and replace $\bm{\Phi}$ by $\mathbf{T}\bm{\Phi}$. By \cref{eq:koopman_k_solution},
\begin{align}
\mathbf{G}_{n,m}&\mapsto\mathbf{T}\mathbf{G}_{n,m}\mathbf{T}^\top,\qquad \mathbf{A}_{n,m}\mapsto\mathbf{T}\mathbf{A}_{n,m}\mathbf{T}^\top,\nonumber\\
\mathbf{K}_{n,m}&\mapsto\mathbf{T}\mathbf{A}_{n,m}\mathbf{G}_{n,m}^{-1}\mathbf{T}^{-1}=\mathbf{T}\mathbf{K}_{n,m}\mathbf{T}^{-1}. \nonumber
\end{align}
An observable $\psi=\mathbf{c}^\top\bm{\Phi}$ has coefficients $\mathbf{T}^{-\top}\mathbf{c}$ in the new basis, and the new coefficient map sends them to $(\mathbf{T}\mathbf{K}_{n,m}\mathbf{T}^{-1})^\top\mathbf{T}^{-\top}\mathbf{c}=\mathbf{T}^{-\top}\mathbf{K}_{n,m}^\top\mathbf{c}$, which represents the same observable $(\mathbf{K}_{n,m}^\top\mathbf{c})^\top\bm{\Phi}$. The operator $\mathcal{K}_{n,m}$ is therefore basis-independent, and so is $\mathcal{K}_n$.
We work in the basis of normalized features $\tilde{\bm{\Phi}}=\mathbf{G}_n^{-1/2}\bm{\Phi}$, which is orthonormal in $\mathcal{H}$, and mark quantities computed from it with a tilde. In this basis,
\begin{align}
&\tilde{\mathbf{G}}_n=\mathbf{I},\qquad \tilde{\mathbf{K}}_n=\tilde{\mathbf{A}}_n,\qquad \|\mathbf{c}^\top\tilde{\bm{\Phi}}\|_{\mathcal{H}}=\|\mathbf{c}\|_2,\nonumber\\
&\|\tilde{\bm{\Phi}}(s)\|_2^2=\bm{\Phi}(s)^\top\mathbf{G}_n^{-1}\bm{\Phi}(s)\le\kappa_n\ \ (s\in\mathcal{S}). \label{eq:whitened_facts}
\end{align}
The identity $\|\mathbf{c}^\top\tilde{\bm{\Phi}}\|_{\mathcal{H}}=\|\mathbf{c}\|_2$ says that the coefficient map is an isometry from $\mathbb{R}^n$ onto $\mathcal{H}_n$, so operator norms on $\mathcal{H}_n$ equal spectral norms $\|\cdot\|_\sigma$ of coefficient matrices. Since $(\mathcal{K}_{n,m}-\mathcal{K}_n)\mathcal{P}_n$ vanishes on $\mathcal{H}_n^\perp$ and $\|\mathcal{P}_n\psi\|_{\mathcal{H}}\le\|\psi\|_{\mathcal{H}}$,
\begin{align}
\epsilon_{\text{samp}}(n,m)=\|\tilde{\mathbf{K}}_{n,m}^\top-\tilde{\mathbf{K}}_n^\top\|_\sigma=\|\tilde{\mathbf{K}}_{n,m}-\tilde{\mathbf{A}}_n\|_\sigma. \label{eq:samp_matrix}
\end{align}
Moreover, $\tilde{\mathbf{A}}_n^\top$ is the coefficient matrix of $\mathcal{K}_n=\mathcal{P}_n\mathcal{K}|_{\mathcal{H}_n}$, so \cref{assumption:koopman_bounded} and $\|\mathcal{P}_n\|\le1$ give
\begin{align}
\|\tilde{\mathbf{A}}_n\|_\sigma=\|\mathcal{P}_n\mathcal{K}|_{\mathcal{H}_n}\|\le\|\mathcal{K}\|\le M. \label{eq:A_bound}
\end{align}

\textit{Step 2 (error decomposition).}
The empirical Koopman matrix depends on estimates of both the Gram and the cross-moment matrices. We separate these two sources of estimation error. Suppose $\tilde{\mathbf{G}}_{n,m}$ is invertible. Adding and subtracting $\tilde{\mathbf{A}}_n\tilde{\mathbf{G}}_{n,m}^{-1}$,
\begin{align}
\tilde{\mathbf{K}}_{n,m}-\tilde{\mathbf{A}}_n&=\tilde{\mathbf{A}}_{n,m}\tilde{\mathbf{G}}_{n,m}^{-1}-\tilde{\mathbf{A}}_n\nonumber\\
&=(\tilde{\mathbf{A}}_{n,m}-\tilde{\mathbf{A}}_n)\tilde{\mathbf{G}}_{n,m}^{-1}+\tilde{\mathbf{A}}_n(\tilde{\mathbf{G}}_{n,m}^{-1}-\mathbf{I}). \label{eq:decomp_K}
\end{align}
Consider the event $\mathcal{E}_\eta=\{\|\tilde{\mathbf{G}}_{n,m}-\mathbf{I}\|_\sigma\le\eta\}$ with $\eta\le\frac12$. Because $\tilde{\mathbf{G}}_{n,m}$ is symmetric, on $\mathcal{E}_\eta$ its eigenvalues lie in $[1-\eta,1+\eta]\subset[\frac12,\frac32]$. Hence $\tilde{\mathbf{G}}_{n,m}$ is invertible, $\|\tilde{\mathbf{G}}_{n,m}^{-1}\|_\sigma\le2$, and
\begin{align}
\|\tilde{\mathbf{G}}_{n,m}^{-1}-\mathbf{I}\|_\sigma&=\|\tilde{\mathbf{G}}_{n,m}^{-1}(\mathbf{I}-\tilde{\mathbf{G}}_{n,m})\|_\sigma\nonumber\\
&\le\|\tilde{\mathbf{G}}_{n,m}^{-1}\|_\sigma\|\mathbf{I}-\tilde{\mathbf{G}}_{n,m}\|_\sigma\le2\|\tilde{\mathbf{G}}_{n,m}-\mathbf{I}\|_\sigma. \nonumber
\end{align}
Taking norms in \cref{eq:decomp_K} and using \cref{eq:samp_matrix,eq:A_bound}, on $\mathcal{E}_\eta$,
\begin{align}
\epsilon_{\text{samp}}(n,m)\le 2\|\tilde{\mathbf{A}}_{n,m}-\tilde{\mathbf{A}}_n\|_\sigma+2M\|\tilde{\mathbf{G}}_{n,m}-\mathbf{I}\|_\sigma. \label{eq:thoery_est_err2}
\end{align}

\textit{Step 3 (concentration).} We bound both deviations with the following lemma.
\begin{lemma}[Matrix Bernstein Inequality~\cite{tropp2012user}]
Let $\{\mathbf{Y}_i\}_{i=1}^m$ be independent, mean-zero $d_1\times d_2$ random matrices with $\|\mathbf{Y}_i\|_\sigma\le L$ almost surely. Let $v^2=\max\{\|\sum_i\mathbb{E}[\mathbf{Y}_i\mathbf{Y}_i^\top]\|_\sigma,\|\sum_i\mathbb{E}[\mathbf{Y}_i^\top\mathbf{Y}_i]\|_\sigma\}$. Then for all $t\ge0$,
$\mathbb{P}(\|\sum_i\mathbf{Y}_i\|_\sigma\ge t)\le(d_1+d_2)\exp\!\big(\frac{-t^2/2}{v^2+Lt/3}\big)$.
\end{lemma}
We first invert the tail bound. Let $\ell'=\log((d_1+d_2)/\delta)$. The right-hand side equals $\delta$ when $t^2/2=\ell'(v^2+Lt/3)$, that is, when $t^2-\tfrac23L\ell'\,t-2v^2\ell'=0$. Its positive root is $t^\star=\frac13L\ell'+\sqrt{\frac19L^2\ell'^2+2v^2\ell'}$. By $\sqrt{a+b}\le\sqrt a+\sqrt b$, $t^\star\le\sqrt{2v^2\ell'}+\frac23L\ell'$. Since the tail bound decreases in $t$, with probability at least $1-\delta$,
\begin{align}
\textstyle\|\sum_i\mathbf{Y}_i\|_\sigma\le\sqrt{2v^2\ell'}+\frac{2}{3}L\ell'. \label{eq:bernstein_inverted}
\end{align}
Both matrices below are $n\times n$, so $d_1+d_2=2n$ and $\ell'=\ell=\log(2n/\delta)$.

\emph{Gram matrix.} Let $\mathbf{Y}_i=\frac1m\big(\tilde{\bm{\Phi}}(s_i)\tilde{\bm{\Phi}}(s_i)^\top-\mathbf{I}\big)$. These matrices are independent by \cref{assumption:iid}, have mean zero because $\mathbb{E}[\tilde{\bm{\Phi}}\tilde{\bm{\Phi}}^\top]=\tilde{\mathbf{G}}_n=\mathbf{I}$, and sum to $\tilde{\mathbf{G}}_{n,m}-\mathbf{I}$. The matrix $\tilde{\bm{\Phi}}\tilde{\bm{\Phi}}^\top-\mathbf{I}$ has the eigenvalue $\|\tilde{\bm{\Phi}}\|_2^2-1\in[-1,\kappa_n-1]$ once and $-1$ otherwise, so $\|\mathbf{Y}_i\|_\sigma\le\kappa_n/m=L$. For the variance, using $\tilde{\bm{\Phi}}^\top\tilde{\bm{\Phi}}=\|\tilde{\bm{\Phi}}\|_2^2\le\kappa_n$,
\begin{align}
\mathbb{E}\big[(\tilde{\bm{\Phi}}\tilde{\bm{\Phi}}^\top-\mathbf{I})^2\big]&=\mathbb{E}\big[\|\tilde{\bm{\Phi}}\|_2^2\,\tilde{\bm{\Phi}}\tilde{\bm{\Phi}}^\top\big]-2\,\mathbb{E}[\tilde{\bm{\Phi}}\tilde{\bm{\Phi}}^\top]+\mathbf{I}\nonumber\\
&\preceq\kappa_n\mathbf{I}-\mathbf{I}\preceq\kappa_n\mathbf{I}, \nonumber
\end{align}
so $v^2\le m\cdot\kappa_n/m^2=\kappa_n/m$ ($\mathbf{Y}_i$ is symmetric, so both variance terms coincide). By \cref{eq:bernstein_inverted}, with probability at least $1-\delta$,
\begin{align}
\|\tilde{\mathbf{G}}_{n,m}-\mathbf{I}\|_\sigma\le\sqrt{\frac{2\kappa_n\ell}{m}}+\frac{2\kappa_n\ell}{3m}. \label{eq:rigorous_bound_G}
\end{align}
Under \cref{eq:stability_threshold}, $\kappa_n\ell/m\le\frac1{12}$, so the right-hand side is at most $\sqrt{2/12}+\frac{2}{36}=\sqrt{1/6}+\frac1{18}\approx0.464<\frac12$. The event $\mathcal{E}_{1/2}$ of Step 2 therefore holds, and $\mathbf{G}_{n,m}=\mathbf{G}_n^{1/2}\tilde{\mathbf{G}}_{n,m}\mathbf{G}_n^{1/2}$ is invertible.

\emph{Cross-moment matrix.} Let $\mathbf{Y}_i=\frac1m\big(\tilde{\bm{\Phi}}(s_i^+)\tilde{\bm{\Phi}}(s_i)^\top-\tilde{\mathbf{A}}_n\big)$, which are independent, have mean zero because $\tilde{\mathbf{A}}_n=\mathbb{E}[\tilde{\bm{\Phi}}(f(s))\tilde{\bm{\Phi}}(s)^\top]$, and sum to $\tilde{\mathbf{A}}_{n,m}-\tilde{\mathbf{A}}_n$. Since $s^+\in\mathcal{S}$, \cref{eq:whitened_facts} gives $\|\tilde{\bm{\Phi}}(s^+)\tilde{\bm{\Phi}}(s)^\top\|_\sigma=\|\tilde{\bm{\Phi}}(s^+)\|_2\|\tilde{\bm{\Phi}}(s)\|_2\le\kappa_n$, and therefore also $\|\tilde{\mathbf{A}}_n\|_\sigma\le\kappa_n$; hence $\|\mathbf{Y}_i\|_\sigma\le2\kappa_n/m=L$. Unlike the Gram matrix, the cross moment contains next-state features. \Cref{assumption:koopman_bounded} bounds their second moments relative to the current-state distribution: for every unit vector $\mathbf{c}$,
\begin{align}
\mathbf{c}^\top\mathbb{E}\big[\tilde{\bm{\Phi}}(s^+)\tilde{\bm{\Phi}}(s^+)^\top\big]\mathbf{c}&=\mathbb{E}_\mu\big[(\mathbf{c}^\top\tilde{\bm{\Phi}}(f(s)))^2\big]=\|\mathcal{K}(\mathbf{c}^\top\tilde{\bm{\Phi}})\|_{\mathcal{H}}^2\nonumber\\
&\le M^2\|\mathbf{c}^\top\tilde{\bm{\Phi}}\|_{\mathcal{H}}^2=M^2, \nonumber
\end{align}
that is, $\mathbb{E}[\tilde{\bm{\Phi}}(s^+)\tilde{\bm{\Phi}}(s^+)^\top]\preceq M^2\mathbf{I}$. Since the variance of a matrix is at most its second moment,
\begin{align}
\sum_i\mathbb{E}[\mathbf{Y}_i\mathbf{Y}_i^\top]&\preceq\frac1m\mathbb{E}\big[\|\tilde{\bm{\Phi}}(s)\|_2^2\,\tilde{\bm{\Phi}}(s^+)\tilde{\bm{\Phi}}(s^+)^\top\big]\nonumber\\
&\preceq\frac{\kappa_n}{m}\mathbb{E}\big[\tilde{\bm{\Phi}}(s^+)\tilde{\bm{\Phi}}(s^+)^\top\big]\preceq\frac{\kappa_nM^2}{m}\mathbf{I}, \nonumber\\
\sum_i\mathbb{E}[\mathbf{Y}_i^\top\mathbf{Y}_i]&\preceq\frac1m\mathbb{E}\big[\|\tilde{\bm{\Phi}}(s^+)\|_2^2\,\tilde{\bm{\Phi}}(s)\tilde{\bm{\Phi}}(s)^\top\big]\nonumber\\
&\preceq\frac{\kappa_n}{m}\mathbb{E}\big[\tilde{\bm{\Phi}}(s)\tilde{\bm{\Phi}}(s)^\top\big]=\frac{\kappa_n}{m}\mathbf{I}. \nonumber
\end{align}
Thus $v^2\le\kappa_n\bar M^2/m$, and by \cref{eq:bernstein_inverted}, with probability at least $1-\delta$,
\begin{align}
\|\tilde{\mathbf{A}}_{n,m}-\tilde{\mathbf{A}}_n\|_\sigma\le\bar M\sqrt{\frac{2\kappa_n\ell}{m}}+\frac{4\kappa_n\ell}{3m}. \label{eq:rigorous_bound_A}
\end{align}

\textit{Step 4 (combination).} By a union bound, \cref{eq:rigorous_bound_G,eq:rigorous_bound_A} hold simultaneously with probability at least $1-2\delta$. On this event, $\mathcal{E}_{1/2}$ holds, and substituting both bounds into \cref{eq:thoery_est_err2} gives
\begin{align}
\epsilon_{\text{samp}}(n,m)\le2(\bar M+M)\sqrt{\frac{2\kappa_n\ell}{m}}+\Big(\frac83+\frac43M\Big)\frac{\kappa_n\ell}{m}. \nonumber
\end{align}
Since $M\le\bar M$ and $\bar M\ge1$, we have $2(\bar M+M)\le4\bar M$ and $\frac83+\frac43M\le4\bar M$, which yields \cref{eq:thoery_est_err3}.
\end{proof}

\begin{remark}[Controlled systems]\label{rem:controlled}
Let the pairs $(x_i,u_i)$ be i.i.d.\ with law $\nu$, let $\mu$ be the marginal law of $x$, and let $x_i^+=f(x_i,u_i)$. The dictionary $\bm{\Phi}$ is fixed independently of the regression samples. The linear-input predictor of \cref{eq:koopman_lin_dyn} is fitted by least squares of $\bm{\Phi}(x^+)$ on the regressor $\mathbf{r}=[\bm{\Phi}(x);u]\in\mathbb{R}^{n+n_u}$. Assume $\mathbb{E}_\nu[\mathbf{r}\mathbf{r}^\top]\succ0$, bounded $\mathbf{r}$ and $\bm{\Phi}$, and that the law $\mu^+$ of $x^+$ satisfies $\mu^+(\mathcal{A})\le M^2\mu(\mathcal{A})$ for every measurable set $\mathcal{A}$, as below \cref{assumption:koopman_bounded}. Let $\kappa_r$ be the leverage constant of $\mathbf{r}$ under $\nu$, $\bar\kappa=\max\{\kappa_n,\kappa_r\}$, and $\ell=\log(2(n+n_u)/\delta)$. If $m\ge12\bar\kappa\ell$, then with probability at least $1-2\delta$ the following holds for the fitted $[\mathbf{A},\mathbf{B}]$. For every $\psi\in\mathcal{H}_n$, the $L_2(\nu)$ distance between its predicted and its population least-squares one-step values is at most $\|\psi\|_{\mathcal{H}}$ times the right-hand side of \cref{eq:thoery_est_err3} with $\kappa_n$ replaced by $\bar\kappa$.
The proof follows the four steps of the proof of \cref{prop:samp_err} and is given in \cref{sec:ap_proof_remark}.
\end{remark}

\paragraph*{What \cref{prop:samp_err} says about scaling}
\textit{(i) Leverage-dependent sampling.} \Cref{prop:samp_err} establishes a sufficient sample-size condition that depends on both the latent dimension and the geometry of the observable space. Call a dictionary \emph{well-spread} if $\kappa_n=\mathcal{O}(n)$; for example, $\kappa_n=n$ for a trigonometric basis under the uniform measure, whereas algebraic polynomials on an interval under the uniform measure have $\kappa_n\sim n^2$~\cite{cohen2013stability}. For well-spread observable spaces, the sampling-error bound is $\mathcal{O}(\sqrt{n\log n/m})$ and is controlled by the ratio $m/(n\log n)$ at fixed confidence. For more concentrated observable spaces, $\kappa_n$ grows faster than $n$, and the sufficient sample requirement increases accordingly. The logarithmic factor is unavoidable for some observable spaces, as the coupon-collector construction shows, but the condition is not necessary for every dictionary. Scaling $m\propto n\log n$ keeps the sampling bound fixed; the bound vanishes when $m/(\kappa_n\log n)\to\infty$. Below the threshold, the empirical Gram matrix can be ill-conditioned, and enlarging the dictionary can increase the error.
\textit{(ii) Leverage versus conditioning.} The leverage constant depends on the observable space rather than its parameterization. An invertible change of feature basis can therefore change the condition number of the Gram matrix without changing $\kappa_n$, and numerical conditioning and statistical sample requirements characterize different properties of a learned representation. The bound $\kappa_n\le B_n^2/\gamma_n$ becomes loose when features are large, or when they are nearly collinear, i.e., when $\gamma_n$ is small.
Conditioning matters for numerical fitting instead. For the fixed-feature, one-step least-squares objective \cref{eq:k_edmd}, the Hessian is proportional to $\mathbf{G}_{n,m}\otimes\mathbf{I}$, so the condition number $\operatorname{cond}(\mathbf{G}_{n,m})=\lambda_{\max}/\lambda_{\min}$ governs the convergence speed of gradient-based fitting and the numerical sensitivity of the fitted matrix. This motivates the covariance regularizer of \cref{sec:method}.

\subsection{Projection Error: What a Latent Dimension Can Represent}
\label{sec:theory_projerr}
Even with unlimited training data, the population Koopman approximation remains restricted to a finite-dimensional observable space. Its accuracy therefore depends on how well this space represents both the observables of interest and their evolution under the true dynamics. For a function $g\in\mathcal{H}$, let $\mathrm{dist}(g,\mathcal{H}_n)=\|(\mathcal{I}-\mathcal{P}_n)g\|_{\mathcal{H}}$. We characterize the two approximation capabilities through the following assumption.
\begin{assumption}[Approximation Rate]\label{assumption:approx}
There exist $C>0$ and $\beta>0$ such that, for every $n$ and every $\psi\in\mathcal{V}$ with $\|\psi\|_{\mathcal{H}}\le1$,
$\mathrm{dist}(\psi,\mathcal{H}_n)\le Cn^{-\beta}$ and $\mathrm{dist}(\mathcal{K}\psi,\mathcal{H}_n)\le Cn^{-\beta}$.
\end{assumption}

\begin{proposition}[Projection Error]\label{prop:proj_err}
Under \cref{assumption:koopman_bounded,assumption:approx},
\begin{align}\label{eq:projection error complex bound}
    \epsilon_{\text{proj}}(n) \le (1+M)\, C\, n^{-\beta}.
\end{align}
\end{proposition}
\begin{proof}
Fix $\psi\in\mathcal{V}$ with $\|\psi\|_{\mathcal{H}}\le1$. Because $\mathcal{K}_n=\mathcal{P}_n\mathcal{K}|_{\mathcal{H}_n}$, we have $\mathcal{K}_n\mathcal{P}_n=\mathcal{P}_n\mathcal{K}\mathcal{P}_n$. Adding and subtracting $\mathcal{P}_n\mathcal{K}$ gives
\begin{align}
\mathcal{K}_n\mathcal{P}_n-\mathcal{K}=\mathcal{P}_n\mathcal{K}\mathcal{P}_n-\mathcal{K}=\mathcal{P}_n\mathcal{K}(\mathcal{P}_n-\mathcal{I})+(\mathcal{P}_n-\mathcal{I})\mathcal{K}.\label{eq:proj_split}
\end{align}
In the first term, $\|\mathcal{P}_n\|\le1$ because $\mathcal{P}_n$ is an orthogonal projection, $\|\mathcal{K}\|\le M$ by \cref{assumption:koopman_bounded}, and $\|(\mathcal{P}_n-\mathcal{I})\psi\|_{\mathcal{H}}=\mathrm{dist}(\psi,\mathcal{H}_n)$. In the second term, $\|(\mathcal{P}_n-\mathcal{I})\mathcal{K}\psi\|_{\mathcal{H}}=\mathrm{dist}(\mathcal{K}\psi,\mathcal{H}_n)$. Hence
\begin{align}
\|(\mathcal{K}_n\mathcal{P}_n-\mathcal{K})\psi\|_{\mathcal{H}}\le M\,\mathrm{dist}(\psi,\mathcal{H}_n)+\mathrm{dist}(\mathcal{K}\psi,\mathcal{H}_n).\label{eq:proj_dist}
\end{align}
The first term arises from approximating the observable before applying the Koopman operator, and the second from approximating its evolution within the finite-dimensional space. \Cref{assumption:approx} bounds both distances by $Cn^{-\beta}$, so $\|(\mathcal{K}_n\mathcal{P}_n-\mathcal{K})\psi\|_{\mathcal{H}}\le(1+M)Cn^{-\beta}$. Taking the supremum over $\psi$ gives \cref{eq:projection error complex bound}.
\end{proof}

The approximation rate $\beta$ depends on the regularity of the observables and of their evolution under the dynamics, as well as on the choice of the finite-dimensional observable spaces. Consider functions with $q$ square-integrable derivatives on a bounded $d$-dimensional domain. The best $n$-dimensional linear subspaces approximate such functions with error of order $n^{-q/d}$ (Kolmogorov $n$-widths~\cite{pinkus1985nwidths}). Non-smooth contact transitions (small $q$) and high-dimensional states (large $d$) therefore suggest slow capacity scaling.
When the state is retained in the embedding ($\mathcal{V}\subset\mathcal{H}_n$), the first condition holds trivially. Then $\epsilon_{\text{proj}}(n)=\sup_{\psi}\mathrm{dist}(\mathcal{K}\psi,\mathcal{H}_n)$: the projection error is exactly the part of the next-state map that the latent features cannot represent linearly.
For $t$-step prediction, the identity $(\mathcal{K}_n\mathcal{P}_n)^t-\mathcal{K}^t=\sum_{k=0}^{t-1}(\mathcal{K}_n\mathcal{P}_n)^{t-1-k}(\mathcal{K}_n\mathcal{P}_n-\mathcal{K})\mathcal{K}^k$ shows that multi-step errors accumulate one-step errors on the observables $\mathcal{K}^k\psi$. These observables need not lie in $\mathcal{V}$, so a multi-step bound needs \cref{assumption:approx} to hold for them as well.
Finally, if $\mathcal{K}$ were compact and normal and $\mathcal{H}_n$ were spanned by its $n$ dominant eigenfunctions, the projection error in operator norm would equal $\sup_{j>n}|\lambda_j|$. We do not adopt such spectral-decay assumptions. If $\mu$ is invariant, $\mathcal{K}$ is an isometry on $L_2(\mu)$ and all its eigenvalues have unit modulus. For other $\mu$, compactness and normality are hard to verify for robot dynamics.

\begin{example}[Fourier observable spaces]\label{ex:fourier}
Let $\mathcal{S}=\mathbb{T}^d$ be the $d$-dimensional torus with the uniform probability measure $\mu$, and let $f$ be a measure-preserving $C^q$ diffeomorphism, $q\ge1$. Then $\mathcal{K}$ is an isometry on $L_2(\mu)$ and $M=1$. Let $\mathcal{H}_n$ be spanned by the constant function and the real trigonometric pairs $\sqrt2\cos(2\pi k^\top s)$, $\sqrt2\sin(2\pi k^\top s)$ with frequencies $0<\|k\|_\infty\le L$, one pair for each $\pm k$, so that $n=(2L+1)^d$. This basis is orthonormal and $\sum_i|e_i(s)|^2=n$ for every $s$, so $\kappa_n=n$, and \cref{assumption:bounded,assumption:gram} hold with $B_n^2=n$ and $\gamma_n=1$. Let $\mathcal{V}$ be a class of observables with $\|\psi\|_{H^q}\le C_\mathcal{V}\|\psi\|_{\mathcal{H}}$ for every $\psi\in\mathcal{V}$, where $H^q$ is the Sobolev space of order $q$; any finite-dimensional space of smooth functions qualifies. Composition with a $C^q$ diffeomorphism is bounded on $H^q(\mathbb{T}^d)$, so $\|\mathcal{K}\psi\|_{H^q}\le C_f\|\psi\|_{H^q}$. For $g\in H^q(\mathbb{T}^d)$, truncating the Fourier series gives $\mathrm{dist}(g,\mathcal{H}_n)\le(2\pi L)^{-q}\|g\|_{H^q}=\mathcal{O}(n^{-q/d})\|g\|_{H^q}$. \Cref{assumption:approx} therefore holds with $\beta=q/d$, and \cref{cor:scaling_law} gives $n^\star\asymp(m/\log m)^{d/(d+2q)}$. The Koopman operator has no decaying spectrum here, yet the data--capacity rule follows from the regularity $q$ of the observables, the dimension $d$ of the domain, and the geometry of the observable space.
\end{example}

\subsection{Error Bound and Data--Capacity Scaling Law}
\begin{theorem}[Error Bound]\label{theorem:error_bound}
Under \cref{assumption:koopman_bounded,assumption:iid,assumption:bounded,assumption:gram,assumption:approx}, let $\delta\in(0,1)$, $\ell=\log(2n/\delta)$, and $m\ge12\kappa_n\ell$. With probability at least $1-2\delta$,
\begin{align}
    \epsilon(n,m) \leq
    \underbrace{4\bar M\Big(\sqrt{\frac{2\kappa_n\ell}{m}}+\frac{\kappa_n\ell}{m}\Big)}_{\text{sampling}}
    + \underbrace{(1+M)\,C\,n^{-\beta}}_{\text{projection}}.
\end{align}
\end{theorem}
\begin{proof}
By \cref{eq:err_decompose}, $\epsilon(n,m)\le\epsilon_\text{samp}(n,m)+\epsilon_\text{proj}(n)$. Under the condition on $m$, \cref{prop:samp_err} bounds $\epsilon_\text{samp}$ by the first term with probability at least $1-2\delta$, and \cref{prop:proj_err} bounds $\epsilon_\text{proj}$ by the second term deterministically. Adding the two bounds on this event gives the result.
\end{proof}

\paragraph*{Controlled systems} For the linear-input predictor of \cref{rem:controlled}, let $(\mathcal{T}\psi)(x,u)=\psi(f(x,u))$, and let $\mathcal{R}_n\subset L_2(\nu)$ be spanned by the components of $\mathbf{r}=[\bm{\Phi}(x);u]$. For $\psi\in\mathcal{H}_n$, the population least-squares predictor of $\mathcal{T}\psi$ is its projection $\mathcal{P}_{\mathcal{R}_n}\mathcal{T}\psi$. For an observable class $\mathcal{V}\subseteq\mathcal{H}_n$, such as the physical-state coordinates retained in the embedding, define the controlled projection error
\begin{align}
\epsilon^\text{ctrl}_\text{proj}(n)\define\sup_{\psi\in\mathcal{V},\,\|\psi\|_{L_2(\mu)}\le1}\big\|(\mathcal{I}-\mathcal{P}_{\mathcal{R}_n})\mathcal{T}\psi\big\|_{L_2(\nu)}.
\end{align}
Under the conditions of \cref{rem:controlled}, with probability at least $1-2\delta$, the one-step error of the fitted predictor on $\mathcal{V}$, measured in $L_2(\nu)$, is at most $4\bar M\big(\sqrt{2\bar\kappa\ell/m}+\bar\kappa\ell/m\big)+\epsilon^\text{ctrl}_\text{proj}(n)$. If $\epsilon^\text{ctrl}_\text{proj}(n)=\mathcal{O}(n^{-\beta})$ and $\bar\kappa=\mathcal{O}(n)$ for a fixed input dimension, the balancing argument of \cref{cor:scaling_law} yields the same data--capacity rule for the linear-input predictor. The linear-input architecture introduces an additional structural restriction: input effects that depend on the state or depend nonlinearly on the input cannot generally be represented exactly. Consequently, the controlled projection error may contain an irreducible component even as $n$ increases.

\begin{corollary}[Data--Capacity Scaling Law]\label{cor:scaling_law}
Suppose the conditions of \cref{theorem:error_bound} hold for all $n$, with $M$, $C$, $\beta$ independent of $n$ and $\kappa_n\le c_0n$. Then for $m\ge12c_0\,n\log(2n/\delta)$, with probability at least $1-2\delta$,
\begin{align}
    \epsilon(n,m) = \mathcal{O}\Big(\sqrt{\frac{n\log(n/\delta)}{m}}\Big) + \mathcal{O}\big(n^{-\beta}\big). \label{eq:converge_sample_err_rate}
\end{align}
Consequently:
(i) Let the data scale as $m=c\cdot n\log n$, with $c$ large enough that the condition on $m$ holds. Then the sampling term stays at $\mathcal{O}(c^{-1/2})$ for every $n$. The coefficient $c$ sets a data-sufficiency floor, and the bound depends on $n$ only through the projection error:
\begin{align}
    \epsilon(n,m)\big|_{m = c\, n \log n} = \mathcal{O}\big(c^{-1/2}\big)+\mathcal{O}\big(n^{-\beta}\big). \label{eq:converge_kmn}
\end{align}
(ii) For a fixed data budget $m$, minimizing a leading upper-bound surrogate yields the latent-dimension scaling rule
\begin{align}
n^\star\asymp\Big(\frac{m}{\log m}\Big)^{\frac{1}{1+2\beta}},\quad \epsilon(n^\star,m)=\mathcal{O}\Big(\Big(\frac{m}{\log m}\Big)^{-\frac{\beta}{1+2\beta}}\Big).\label{eq:optimal_n}
\end{align}
Latent dimensions much larger than $n^\star$ increase the bound. This rule balances the leading terms of the bound; the best latent dimension of a trained model can differ, because the bound is conservative and does not describe the optimization and representation-learning errors of end-to-end training.
(iii) The sampling and projection terms of the upper bound have the same asymptotic order when $m\asymp n^{1+2\beta}\log n$, which is faster than $n\log n$.
\end{corollary}
\begin{proof}
Since $\kappa_n\le c_0n$, the condition $m\ge12c_0n\ell$ implies $m\ge12\kappa_n\ell$, so \cref{theorem:error_bound} applies. On this range $\kappa_n\ell/m\le\frac1{12}$, hence $\kappa_n\ell/m\le\sqrt{\kappa_n\ell/m}/\sqrt{12}$, and the sampling term satisfies
\begin{align*}
4\bar M\Big(\sqrt{\frac{2\kappa_n\ell}{m}}+\frac{\kappa_n\ell}{m}\Big)\le4\bar M\Big(\sqrt2+\frac{1}{\sqrt{12}}\Big)\sqrt{\frac{c_0n\ell}{m}},
\end{align*}
which gives \cref{eq:converge_sample_err_rate}.
(i) With $m=c\,n\log n$ and $n\ge2$,
\begin{align*}
\frac{n\log(n/\delta)}{m}=\frac{1}{c}\Big(1+\frac{\log(1/\delta)}{\log n}\Big)\le\frac{1}{c}\Big(1+\frac{\log(1/\delta)}{\log2}\Big),
\end{align*}
so the sampling term is $\mathcal{O}(c^{-1/2})$ uniformly in $n$.
(ii) For fixed $m$ and $n\le m$, $\log(n/\delta)\le\log(m/\delta)$, so the bound is at most $g(n)=a\sqrt{n\log m/m}+b\,n^{-\beta}$ for constants $a,b>0$. Setting $g'(n)=0$ gives
\begin{align*}
\frac a2\sqrt{\frac{\log m}{mn}}=\beta b\,n^{-\beta-1}\quad\Longleftrightarrow\quad n^{1/2+\beta}=\frac{2\beta b}{a}\sqrt{\frac{m}{\log m}},
\end{align*}
so $n^\star\asymp(m/\log m)^{1/(1+2\beta)}$. At $n^\star$ both terms are of order $(m/\log m)^{-\beta/(1+2\beta)}$. For $n\gg n^\star$ the sampling term dominates and grows like $\sqrt n$, so the bound increases.
(iii) The two terms are of the same order when $\sqrt{n\log n/m}\asymp n^{-\beta}$, that is, when $m\asymp n^{1+2\beta}\log n$.
\end{proof}

\Cref{cor:scaling_law} distinguishes three regimes under its assumptions: $m\asymp n\log n$ keeps the sampling bound fixed, $n^\star\asymp(m/\log m)^{1/(1+2\beta)}$ selects the latent dimension for a fixed data budget, and $m\asymp n^{1+2\beta}\log n$ balances the decay of both terms. These are statements about an upper bound, not lower bounds or the exact empirically optimal dimension, which also depends on representation learning and optimization. The same balancing applies to observable spaces that are not well spread. If $\kappa_n=\mathcal{O}(n^p)$ with $p\ge1$ and the other assumptions hold, the sampling term is $\mathcal{O}(\sqrt{n^p\log m/m})$, and balancing it against $n^{-\beta}$ gives
\begin{align}
n^\star\asymp\Big(\frac{m}{\log m}\Big)^{\frac{1}{p+2\beta}}. \label{eq:optimal_n_p}
\end{align}
The useful latent dimension thus depends on two distinct properties of the observable spaces: the approximation rate $\beta$ and the leverage growth $p$. For the same data, observable spaces whose leverage grows faster support smaller latent dimensions.

\begin{corollary}[Strong Consistency]\label{cor:convergence}
Suppose \cref{assumption:koopman_bounded} holds, and \cref{assumption:iid,assumption:bounded,assumption:gram} hold for each $n$. Suppose the subspaces satisfy $\lim_{n\to\infty}\mathrm{dist}(\psi,\mathcal{H}_n)=0$ for every $\psi\in\mathcal{H}$, e.g., nested subspaces with dense union. If $m_n/(\kappa_n\log n)\to\infty$, then for every $\psi\in\mathcal{H}$,
\begin{align}
    \big\|(\mathcal{K}_{n,m_n}\mathcal{P}_n-\mathcal{K})\psi\big\|_{\mathcal{H}}\ \longrightarrow\ 0 \quad \text{almost surely as } n\to\infty.
\end{align}
\end{corollary}
\begin{proof}
Fix $\psi\in\mathcal{H}$. By the triangle inequality,
\begin{align}
\|(\mathcal{K}_{n,m_n}\mathcal{P}_n-\mathcal{K})\psi\|_{\mathcal{H}}&\le\|(\mathcal{K}_{n,m_n}-\mathcal{K}_n)\mathcal{P}_n\psi\|_{\mathcal{H}}\nonumber\\
&\quad+\|(\mathcal{K}_n\mathcal{P}_n-\mathcal{K})\psi\|_{\mathcal{H}}. \label{eq:consistency_split}
\end{align}
\textit{Sampling term.} Apply \cref{prop:samp_err} with confidence $\delta_n=n^{-2}$. Then $\ell_n=\log(2n/\delta_n)=\log(2n^3)\le4\log n$ for $n\ge2$. Because $m_n/(\kappa_n\log n)\to\infty$, there is an $n_0$ such that $m_n\ge48\,\kappa_n\log n\ge12\kappa_n\ell_n$ for all $n\ge n_0$, so the condition \cref{eq:stability_threshold} holds. Let $\mathcal{F}_n$ be the event that the bound \cref{eq:thoery_est_err3} fails at index $n$. Then $\mathbb{P}(\mathcal{F}_n)\le2n^{-2}$ for $n\ge n_0$, and $\sum_n\mathbb{P}(\mathcal{F}_n)<\infty$. By the Borel--Cantelli lemma, with probability one only finitely many $\mathcal{F}_n$ occur. Hence, almost surely, for all sufficiently large $n$,
\begin{align}
\|(\mathcal{K}_{n,m_n}-\mathcal{K}_n)\mathcal{P}_n\psi\|_{\mathcal{H}}&\le\epsilon_\text{samp}(n,m_n)\|\psi\|_{\mathcal{H}}\nonumber\\
&\le4\bar M\Big(\sqrt{\tfrac{8\kappa_n\log n}{m_n}}+\tfrac{4\kappa_n\log n}{m_n}\Big)\|\psi\|_{\mathcal{H}}, \nonumber
\end{align}
and the right-hand side tends to zero because $\kappa_n\log n/m_n\to0$.
\textit{Projection term.} The bound \cref{eq:proj_dist} from the proof of \cref{prop:proj_err} holds for every $\psi\in\mathcal{H}$, since it uses only \cref{eq:proj_split} and \cref{assumption:koopman_bounded}. Both distances tend to zero by the density assumption, applied to $\psi$ and to $\mathcal{K}\psi\in\mathcal{H}$. This term is deterministic.
Combining the two terms in \cref{eq:consistency_split} proves the claim.
\end{proof}
This is convergence for each fixed observable (strong operator topology), as in~\cite{korda2018convergence}. It does not imply $\|\mathcal{K}_{n,m}\mathcal{P}_n-\mathcal{K}\|_{\mathcal{L}(\mathcal{H})}\to0$, which fails whenever $\mathcal{H}$ is infinite-dimensional and $\mu$ is invariant. The density condition is an assumption on the learned subspaces. The universal approximation theorem~\cite{hornik1989multilayer,cybenko1989approximation} shows that a network class can represent any observable, but not that training selects subspaces that become dense.

\subsection{Representation-Learning Gap}
A small prediction loss does not uniquely determine the learned observable space. Different embeddings can reach similar accuracy while differing in conditioning, leverage constant, and the structure relevant to control.
Let $\mathcal{H}_n(\theta)$ be the subspace spanned by an encoder with parameters $\theta\in\Theta_n$, let $\epsilon_{\text{proj}}(n;\theta)$ be its projection error, and let $\hat\theta$ be the trained parameters. Define the representation-learning gap
\begin{align}
\epsilon_{\text{learn}}(\hat\theta)\define\epsilon_{\text{proj}}(n;\hat\theta)-\inf_{\theta\in\Theta_n}\epsilon_{\text{proj}}(n;\theta)\ \ge0.
\end{align}
\begin{corollary}[Scaling with Learned Observable Spaces]\label{cor:realizable_scaling}
Suppose \cref{assumption:koopman_bounded} holds and $\inf_{\theta\in\Theta_n}\epsilon_{\text{proj}}(n;\theta)\le(1+M)Cn^{-\beta}$. Suppose further that the encoder is learned independently of the $m$ samples used to fit the transition matrix, and that, almost surely over encoder training, the trained features satisfy \cref{assumption:iid,assumption:bounded,assumption:gram} with leverage constant $\kappa_n\le c_0n$. Then, whenever $m\ge12c_0\,n\log(2n/\delta)$, with probability at least $1-2\delta$,
\begin{align}
    \epsilon(n, m) &\le \underbrace{\mathcal{O}\Big(\sqrt{ \frac{n\log (n/\delta)}{m} }\Big)}_{\text{sampling error}} + \underbrace{\mathcal{O}\big(n^{-\beta}\big)}_{\text{projection error}} \nonumber\\
    &\quad+ \underbrace{\epsilon_\text{learn}(\hat\theta)}_{\text{representation-learning gap}}.
    \label{eq:total_error_with_opt}
\end{align}
\end{corollary}
\begin{proof}
Condition on the trained parameters $\hat\theta$. Because the encoder is learned independently of the $m$ regression samples, the dictionary $\bm{\Phi}_{\hat\theta}$ is fixed with respect to these samples, and \cref{assumption:iid} holds conditionally. Apply \cref{eq:err_decompose} to the learned space $\mathcal{H}_n(\hat\theta)$. For the sampling term, \cref{prop:samp_err} with $\kappa_n\le c_0n$ gives, as in the proof of \cref{cor:scaling_law}, $\epsilon_\text{samp}(n,m)=\mathcal{O}(\sqrt{n\log(n/\delta)/m})$ with probability at least $1-2\delta$ whenever $m\ge12c_0n\log(2n/\delta)$. For the projection term, the definition of the representation-learning gap gives
\begin{align}
\epsilon_\text{proj}(n;\hat\theta)&=\inf_{\theta\in\Theta_n}\epsilon_\text{proj}(n;\theta)+\epsilon_\text{learn}(\hat\theta)\nonumber\\
&\le(1+M)Cn^{-\beta}+\epsilon_\text{learn}(\hat\theta). \nonumber
\end{align}
Adding the two bounds gives \cref{eq:total_error_with_opt}. Since the bound holds for every $\hat\theta$, it also holds unconditionally.
\end{proof}
The representation-learning gap measures the difference between the approximation error of the learned observable space and the best approximation error achievable within the encoder class. It may arise from finite training data, imperfect optimization, or a mismatch between the training objective and the population approximation metric. It need not vanish as $m$ and $n$ grow, and it may contribute to the error floors in the empirical scaling curves (\cref{fig:scaling_samples}). If the encoder is trained on the same samples used to fit the transition matrix, a uniform generalization bound over the encoder class is also needed. We leave this extension to future work.

\section{Practical Neural Koopman Operators}
\label{sec:method}

This section describes the model, the training objective, and the controller used in our experiments.

\subsection{Koopman Model Architecture}
We use the embedding of \cref{sec:koopman_operator_theory}. An encoder $\Psi_\theta:\mathbb{R}^{n_x}\rightarrow\mathbb{R}^{n_\text{enc}}$ outputs $n_\text{enc}$ nonlinear features, and the latent state concatenates them with the physical state:
\begin{equation}
    z_t = \Phi_\theta(x_t) := [x_t; \Psi_\theta(x_t)] \in \mathbb{R}^{n},\qquad n = n_x + n_\text{enc}. \label{eq:koopman_arch}
\end{equation}
The latent state evolves as $z_{t+1} = \mathbf{A} z_t + \mathbf{B} u_t$, the time-indexed form of \cref{eq:koopman_lin_dyn}, where $\mathbf{A}$ and $\mathbf{B}$ are bias-free linear layers. We initialize $\mathbf{A}$ near-orthogonal by a singular value decomposition, and the last encoder layer with orthogonal weights.
We model control inputs through a constant latent input matrix $\mathbf{B}$, so that the downstream MPC problem remains a convex quadratic program. This is a modeling approximation, not a consequence of the physics: even a control-affine plant $x_{t+1}=f(x_t)+g(x_t)u_t$ generally has state-dependent input effects after lifting. The constant-$\mathbf{B}$ structure trades representational flexibility for efficient real-time optimization.

\subsection{Training Objective} \label{sec:train_objective}
The theoretical analysis highlights the importance of observable-space geometry for finite-sample estimation. In practical neural Koopman learning, the numerical conditioning of the feature representation and its ability to capture control effects are also important. We therefore augment the multi-step prediction objective with two complementary regularizers. The covariance loss discourages redundant latent features, while the inverse-control loss encourages latent transitions to preserve information about the applied control inputs. These objectives target properties of the learned representation that are not directly captured by physical-state prediction error. The model is trained end-to-end with the composite loss
\begin{equation}
\mathcal{L}_{\text{total}} = \mathcal{L}_{\text{pred}} + w_{\text{cov}} \mathcal{L}_{\text{cov}} + w_{\text{ctrl}} \mathcal{L}_{\text{ctrl}},
\end{equation}
where $w_{\text{cov}}$ and $w_{\text{ctrl}}$ are tunable scalar weights for the covariance and inverse-control losses, respectively.

\paragraph{Multi-Step Prediction Loss.}
We unroll the learned linear model over a horizon of $T$ steps, starting from the encoded state $z_t = \Phi_\theta(x_t)$:
\begin{align}
    \hat{z}_{\tau+1} &= \mathbf{A} \hat{z}_\tau + \mathbf{B} u_\tau, \text{~for~} \tau = t,\dots,t+T-1, \quad \hat{z}_t = z_t.
\end{align}
The projection $\mathbf{P}$ maps each predicted latent state to the predicted physical state $\hat{x}_\tau = \mathbf{P}\hat{z}_\tau$. The multi-step prediction loss is the discounted mean squared error (MSE) of these predictions:
\begin{equation}
    \mathcal{L}_{\text{pred}}
    =
    \frac{1}{\Gamma}
    \sum_{\tau=t+1}^{t+T}
    \rho^{\tau-t-1}
    \,
    \|\hat{x}_\tau - x_\tau\|_2^2,
\end{equation}
where $\rho \in (0,1]$ is a temporal discount factor that emphasizes near-term prediction accuracy, and
$\Gamma = \sum_{j=0}^{T-1} \rho^j$ is a normalization constant.

\paragraph{Covariance Loss.}
As the latent dimension grows, learned features may become highly correlated, which makes the regression problem poorly conditioned. Although the statistical sampling bound depends on the observable space only through its leverage constant, numerical conditioning also affects the estimation and optimization of the transition matrix (\cref{sec:sup_proof}). We therefore regularize the covariance of the learned features $h_i=\Psi_\theta(x_i)\in\mathbb{R}^{n_\text{enc}}$. For a training batch of $b$ initial states, we compute the empirical covariance matrix $\boldsymbol{\Sigma}_\text{batch} \in \mathbb{R}^{n_\text{enc} \times n_\text{enc}}$:
\begin{align}
    \boldsymbol{\Sigma}_\text{batch} = \frac{1}{b-1} \sum_{i=1}^b \big(h_i - \bar{h}\big)\big(h_i - \bar{h}\big)^\top,
\end{align}
where $\bar{h} = \frac{1}{b}\sum_{i=1}^b h_i$ is the mean feature vector. The covariance loss penalizes its off-diagonal entries:

\begin{equation}
\mathcal{L}_{\text{cov}} = \frac{1}{n_\text{enc}(n_\text{enc}-1)}
\left\| \boldsymbol{\Sigma}_\text{batch} - \mathrm{diag}(\boldsymbol{\Sigma}_\text{batch}) \right\|_F^2. \label{eq:l_cov_def}
\end{equation}
The normalization expresses the loss as the mean squared off-diagonal covariance, which makes its magnitude more comparable across latent dimensions. The loss discourages redundant features, which improves conditioning when feature variances stay balanced. It does not by itself bound the smallest covariance eigenvalue, and the centered covariance differs from the uncentered Gram matrix of \cref{sec:sup_proof}. We therefore evaluate its effect on conditioning empirically (\cref{sec:exp_cov_loss_cond}).

\paragraph{Inverse Control Loss.}

Accurate state prediction does not ensure that the learned input matrix $\mathbf{B}$ is informative about the applied controls. Minimizing $\mathcal{L}_{\text{pred}}$ only constrains the composite transition $(\mathbf{A},\mathbf{B})$ to match observed trajectories. A predictor can reach low state-prediction error while learning a latent input map that is insensitive to some input directions. Such a representation can be less useful for control, even when its predictions are accurate on the training distribution.

We therefore add an inverse-control loss, which requires the applied input $u_\tau$ to be recoverable from the latent transition residual $(z_{\tau+1} - \mathbf{A} z_\tau)$, where $z_\tau=\Phi_\theta(x_\tau)$ are encodings of measured states. We estimate the input with the least-squares pseudoinverse
\begin{align}
    \hat{u}_\tau &= (\mathbf{B}^\top \mathbf{B})^\dagger \mathbf{B}^\top
    \big(z_{\tau+1} - \mathbf{A} z_\tau\big), \label{eq:inv_control}
\end{align}
and the inverse-control loss penalizes the reconstruction error:
\begin{align}
    \mathcal{L}_{\text{ctrl}}  &=  \frac{1}{\Gamma}  \sum_{\tau=t}^{t+T-1}  \rho^{\tau-t} \|\hat{u}_\tau - u_\tau\|_2^2.
\end{align}

The inverse-control loss encourages latent transitions to retain information about the applied inputs, which complements state prediction and yields a representation more useful for control. It does not by itself establish controllability or guarantee closed-loop stability; we evaluate its effect empirically.

\subsection{Koopman-based Model Predictive Control} \label{sec:koopman_mpc_form}
To evaluate the control performance of the learned operators, we embed them within a linear Model Predictive Control (MPC) framework. The learned Koopman model serves as the internal dynamics model. At each time step $t$, we solve a finite-horizon quadratic program (QP) to compute an optimal sequence of control inputs.

Unlike latent-space MPC formulations that penalize the full lifted state, we define the cost directly in the original (measurable) state space to prioritize physical tracking accuracy. Let $x_\tau$ denote the predicted physical state at time $\tau$ over a horizon of length $H$. The optimization problem is formulated as
\begin{align}
\min_{\{u_\tau\}_{\tau=t}^{t+H-1}}
& \sum_{\tau=t}^{t+H-1}
(x_{\tau+1} - x_{\tau+1}^{\mathrm{ref}})^\top Q
(x_{\tau+1} - x_{\tau+1}^{\mathrm{ref}})
+
u_\tau^\top R u_\tau \nonumber \\
\text{s.t.} \quad
& z_{\tau+1} = \mathbf{A} z_\tau + \mathbf{B} u_\tau,
\quad \tau = t,\dots,t+H-1, \nonumber \\
& z_{t} = \Phi(x_t), ~ x_{\tau} = \mathbf{P} z_{\tau}. \label{eq:koopman_mpc}
\end{align}
The learned linear dynamics evolve in the latent space, while physical states used in the objective are obtained through the projection operator $x_\tau = \mathbf{P} z_\tau$. This formulation ensures that the controller optimizes performance in the true state space rather than regulating latent variables.
We solve the unconstrained finite-horizon quadratic problem in closed form by least squares. At each time step, we clip the first planned input $u_t^\star$ to the actuator limits $u_{\min} \le u \le u_{\max}$, apply it, and repeat at the next step in a receding-horizon manner.

Following \cite{li2025continuallearningliftingkoopman}, we set $Q = I$ (consistent with the MSE-based training objective, without temporal discounting) and, by default, $R = 0$. The tuned closed-loop studies select $R=r\mathbf{I}$ with $r\in\{0,10^{-3}\}$ on validation trajectories (\cref{sec:ap_aux_tuning}). This choice emphasizes physical tracking accuracy while not explicitly penalizing control effort; the QP remains convex, although its minimizer need not be unique when $R=0$, in which case the closed-form solution selects one minimizer. Its solution is optimal for the learned finite-horizon model, not necessarily for the nonlinear plant. We leave joint design of the training loss and MPC cost as future work.
\section{Experimental Design}
\label{sec:experiments}

We test whether the trends predicted by \cref{sec:sup_proof} appear in jointly trained neural Koopman models, which go beyond the fixed-dictionary setting of the theory. We ask two questions. \textbf{1) Data--capacity scaling:} How does held-out prediction error change with dataset size ($m$) and latent dimension ($n$), and does increasing $n$ help only when the data can support it (\cref{cor:scaling_law})? \textbf{2) Regularization for control:} Do the auxiliary losses ($\mathcal{L}_\text{cov}$ and $\mathcal{L}_\text{ctrl}$) improve closed-loop tracking robustness and task completion, beyond their effect on open-loop prediction accuracy?
We study these questions by analyzing prediction error scaling across varying data regimes and model sizes, followed by closed-loop control evaluations and ablations of the covariance and inverse-control losses.

\subsection{Data Collection Strategy}
To evaluate robustness across diverse dynamical systems, we adopt three data collection strategies:

\textbf{Strategy I (Uniform State-Space Sampling):} For stable synthetic systems, we uniformly sample initial states and control inputs to achieve broad, unbiased coverage of the global dynamics.
\textbf{Strategy II (Randomized Simulation Actuation):}  For simulated manipulators, we generate trajectories via randomized control (e.g., joint velocity commands), enabling diverse kinematic exploration without a predefined policy.
\textbf{Strategy III (Constrained \& Expert-Guided Exploration):}  For unstable locomotion, where random actuation quickly leads to falls, we aggregate expert demonstrations\footnote{Expert policies are utilized solely for safe data generation; the resulting linear Koopman model is amenable to linear-systems analysis and enables computationally efficient MPC for objectives beyond the expert's original task.} with on-policy data to safely cover the necessary state distribution.

\subsection{Experimental Setup}
We evaluate our approach on robotic and mechanical systems of varying complexity; \cref{tab:environments} summarizes their state and control dimensions.

\begin{table}[ht]
    \centering
    \vspace{-5pt}
    \caption{Summary of environment dimensions.
    }
    \label{tab:environments}
    \vspace{-5pt}
    \begin{tabular}{lcc}
        \toprule
        \textbf{Environment} & \textbf{State Dim.} & \textbf{Action Dim.} \\
        \midrule
        Damping Pendulum & 2 &  1  \\
        Double Pendulum &  4 &  2  \\
        Franka Panda    & 14 &  7  \\
        Kinova Gen3     & 14 &  7  \\
        Unitree Go2     & 35 & 12  \\
        Unitree G1      & 53 & 23  \\
        \bottomrule
    \end{tabular}
    \vspace{-5pt}
\end{table}

The damped and double pendulums use uniform state-space sampling with randomized inputs (Strategy I). The simulated 7-DoF Franka Panda arm, whose low-level joint controller makes the measured velocities differ from the commanded ones, uses randomized joint-velocity commands (Strategy II). The Unitree Go2 quadruped and G1 humanoid combine a pretrained RL expert policy with on-policy recovery data from Koopman-based MPC (Strategy III). The real Kinova Gen3 arm uses trajectory optimization between random configurations with collision and joint-limit constraints (Strategy III). \Cref{sec:ap_exp_setup} gives the details.

\subsection{Evaluation Metrics}

We evaluate performance via \emph{open-loop prediction accuracy} on held-out test sets and \emph{closed-loop control}. Open-loop prediction error ($\mathcal{L}_{\text{pred}}$) is measured across all tasks to assess scaling and generalization, evaluating the learned dynamics independently of feedback. Tracking experiments focus on our most challenging high-dimensional platforms, the Unitree G1 and Go2 robots. We use \emph{Mean Survival Steps} \cite{li2025continuallearningliftingkoopman}: the average number of consecutive timesteps of successful locomotion tracking before the tracking error exceeds a prescribed threshold (mean joint-position error $>0.10$ rad for G1, $>0.16$ rad for Go2). This metric directly captures closed-loop robustness in unstable, contact-rich regimes where small errors accumulate into catastrophic failures. For the Franka arm, we report the end-effector tracking error on a figure-8 reference trajectory \cite{shi2022deep}.

\subsection{Model Architecture and Training Configuration}

\textbf{Training samples.} We denote the number of training samples by $m$. In the neural Koopman experiments, each training sample is a trajectory segment of $T=15$ transitions, and we vary $m \in \{1, 4, 16, 64, 140\}$k. In the theory and in the fixed-feature experiment of \cref{sec:controlled_grid}, each sample is one independently drawn transition. For all systems except Kinova, each segment is generated separately, so the segments are close to i.i.d., although the transitions within a segment are not. The Kinova segments are cut from five long recordings and split chronologically, with gaps so that no time step is shared between splits (\cref{sec:ap_exp_setup}). With Strategy III, the data come from a mixture of collection policies, including earlier Koopman controllers, so the sampling distribution $\mu$ is this mixture rather than a fixed design.

\textbf{Model and training.} The encoder $\Psi_\theta$ is an MLP with three 256-unit hidden layers, ReLU activations, and residual connections. We scale the latent dimension with the multiplier $n_{\text{mult}}:=n_\text{enc}/n_x \in \{1, 2, 4, 8, 16\}$, so that $n = (n_{\text{mult}} + 1) n_x$. Each configuration of the scaling sweeps (\cref{fig:scaling_samples,fig:scaling_encodedim,tab:scaling_fits}) is trained with five random seeds and the prediction loss only. The coupled-scaling study (\cref{fig:mlogn_scaling}), the closed-loop studies, and the conditioning diagnostics use separate training runs, so the same nominal configuration can have slightly different errors across figures and tables. \Cref{sec:ap_impl} lists all training and control settings.

\section{Empirical Analysis}

\begin{figure*}[t]
    \centering
    \vspace{-5pt}
    \includegraphics[width=\textwidth]{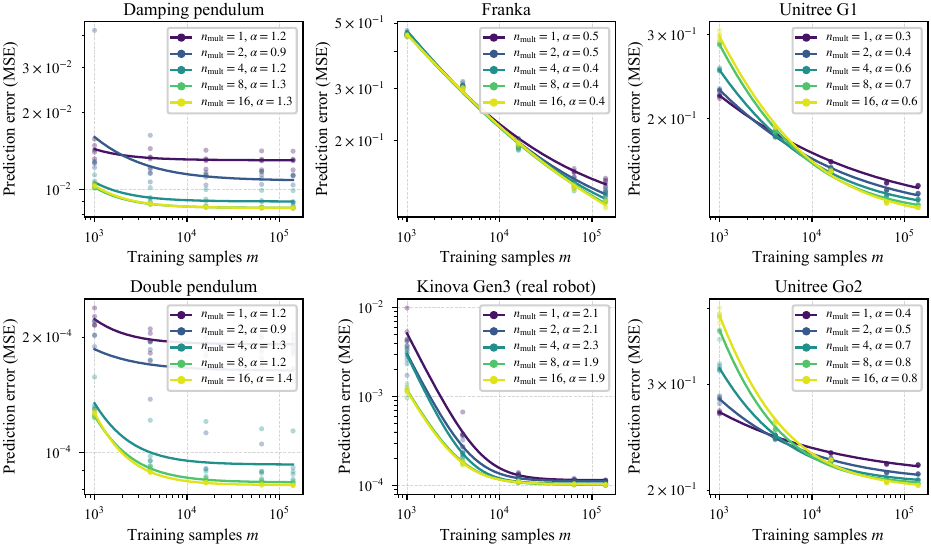}
    \vspace{-5pt}
    \caption{ Scaling of prediction error with the number of training samples $m$. The error is the held-out multi-step prediction loss $\mathcal{L}_\text{pred}$ of the physical state (normalized states for Franka, G1, and Go2). Subplots display error curves across latent dimension multipliers $n_{\text{mult}}$ ($n=(n_\text{mult}+1)\,n_x$). Dots denote random seeds, and solid lines show fitted power-law trends. The legend reports the fitted empirical exponent $\alpha$. Increasing dataset size consistently reduces error across environments, at rates that differ substantially between systems.}
    \label{fig:scaling_samples}
    \vspace{-5pt}
\end{figure*}

\subsection{Scaling Laws of Koopman Models}
Our theoretical analysis (\cref{theorem:error_bound,cor:realizable_scaling}) decomposes the total approximation error into sampling ($\epsilon_{\text{samp}}$) and projection ($\epsilon_{\text{proj}}$) errors, plus a representation-learning gap, with the latent dimension $n$ entering both of the first two terms. We measure the held-out discounted multi-step MSE of the physical state ($\mathcal{L}_\text{pred}$ on the test split), the multi-step counterpart of the one-step quantity bounded in \cref{eq:one_step_mse}. Two differences from the theory matter when reading the fits. First, a sampling error of order $m^{-1/2}$ corresponds to a squared error of order $m^{-1}$ only in the fixed-feature, one-step setting; our curves measure multi-step MSE with learned features and a nonzero approximation error. Second, the encoder is retrained for every $(m,n)$, so more data also improves the learned features. The fitted exponents are therefore descriptive measurements over the tested range, not estimates of the constants in \cref{theorem:error_bound}. We first scale data and latent dimension separately and then together (\cref{sec:unified_scaling}). For each sweep, we fit the prediction error $\epsilon$ to a power law:
\begin{equation}
\epsilon(D) = a \cdot D^{-\alpha} + \epsilon_\infty, \label{eq:practical_scaling_law}
\end{equation}
where $D$ is the scaled resource (dataset size $m$ or latent dimension $n$), $a$ is the decay magnitude, and $\epsilon_\infty$ is the fitted asymptotic error floor.

\subsubsection{Scaling with Training Sample Size}
\Cref{prop:samp_err} predicts that the error decreases with data at a rate governed by the ratio $m/(\kappa_n\log n)$. Fitting \cref{eq:practical_scaling_law} with $D=m$ gives the data exponent $\alpha_m$ (\cref{fig:scaling_samples,tab:scaling_fits}; brackets give seed-bootstrap intervals). We observe three regimes.

\noindent \textit{1) Slow-Improvement Regime (Franka, G1, Go2):} For the high-dimensional systems, error decays slowly with data ($\alpha_m \in [0.44, 0.65]$), well below the $m^{-1}$ MSE rate of the fixed-dictionary sampling term. A possible explanation is that more data mainly helps the encoder learn a better representation, which would be consistent with large projection and representation-learning components; contact-rich, hybrid dynamics are a plausible source of this difficulty for G1 and Go2.

\noindent \textit{2) Rapid Convergence Regime (Pendulums, Kinova):} The pendulums improve much faster with data ($\alpha_m \approx 1.2$). The real Kinova arm is fastest ($\alpha_m \approx 2.0$). A plausible reason is that its data come from trajectory optimization, which keeps them on a low-dimensional set of smooth, locally optimal paths, unlike the broad coverage of random exploration. On Kinova, the error drops by a factor of 10 to 40 up to $m=16$k and then stays at a floor of about $1.0\times10^{-4}$, plausibly set by sensor noise. Exponents above one do not contradict \cref{prop:samp_err}, because the learned features also improve with $m$.

\noindent \textit{3) Data--Capacity Crossover:} The trade-off predicted by \cref{cor:scaling_law} appears in the G1 and Go2 curves (\cref{fig:scaling_samples}, right column). At $m\approx10^3$, larger latent spaces predict worse, and with sufficient data they predict better (\cref{fig:scaling_encodedim}). This is the regime $n\gg n^\star(m)$ of \cref{cor:scaling_law}, whose bound grows with $n$ when $m$ is small relative to $\kappa_n\log n$; overfitting of the encoder, which the theory does not model, may also contribute.

\begin{figure*}[t]
    \centering
    \vspace{-5pt}
    \includegraphics[width=\textwidth]{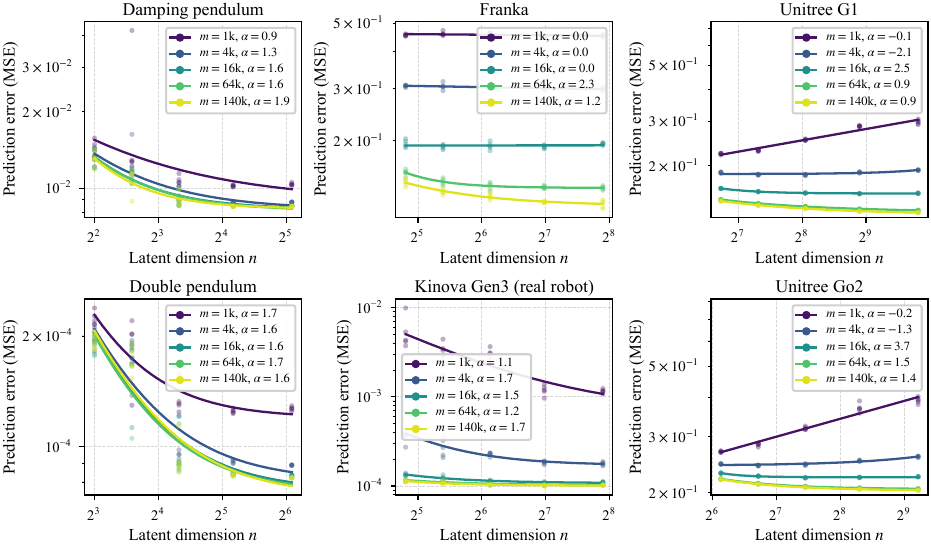}
    \vspace{-5pt}
    \caption{ Scaling of prediction error with latent dimension $n=(n_\text{mult}+1)\,n_x$, for the same runs as \cref{fig:scaling_samples}. Subplots display error curves across training-set sizes $m$. With sufficient data, increasing latent dimension reduces error; in low-data regimes, larger embeddings can predict worse. For G1 and Go2, the curves for $m=1$k and $4$k rise with $n$.}
    \label{fig:scaling_encodedim}
    \vspace{-5pt}
\end{figure*}

\begin{table}[ht]
\centering
\vspace{-5pt}
\caption{ Fitted empirical power-law exponents for prediction MSE. Each fit pools five seeds. The data exponent $\alpha_m$ is averaged over the five latent multipliers. The capacity exponent $\alpha_n$ is averaged over the data sizes $m \ge 16$k, where finite-sample overfitting is small. Brackets: 95\% intervals from 300 seed-bootstrap resamples. These are descriptive fits, not measured Koopman spectra.}
\label{tab:scaling_fits}
\vspace{-5pt}
\begin{tabular}{lcc}
\toprule
\textbf{Environment} & \textbf{$\alpha_m$ (Sample Size)} & \textbf{$\alpha_n$ (Latent Dimension)} \\
\midrule
Damping Pendulum & 1.16 {\scriptsize[0.98, 1.26]} & 1.73 {\scriptsize[1.31, 2.28]} \\
Double Pendulum  & 1.21 {\scriptsize[1.00, 1.32]} & 1.62 {\scriptsize[1.39, 2.05]} \\
Franka           & 0.44 {\scriptsize[0.43, 0.44]} & 1.16 {\scriptsize[$-0.31$, 10.5]} \\
Kinova           & 2.03 {\scriptsize[1.99, 2.07]} & 1.46 {\scriptsize[1.16, 1.79]} \\
Go2              & 0.65 {\scriptsize[0.64, 0.66]} & 2.16 {\scriptsize[2.07, 2.26]} \\
G1               & 0.52 {\scriptsize[0.52, 0.52]} & 1.44 {\scriptsize[1.36, 1.55]} \\
\bottomrule
\end{tabular}
\vspace{-5pt}
\end{table}

\subsubsection{Scaling with Model Size}
\Cref{prop:proj_err} bounds the projection error by $(1+M)Cn^{-\beta}$, where $\beta$ reflects how regular the next-state map is relative to the state dimension. Fitting \cref{eq:practical_scaling_law} with $D=n$ gives the capacity exponent $\alpha_n$ (\cref{fig:scaling_encodedim,tab:scaling_fits}).

\noindent \textit{1) Capacity Helps When the Data Suffice:} With $m \ge 64$k, a larger latent dimension reduces the error on every system; on the pendulums and Kinova it helps already at $m=1$k. The capacity exponents lie between $1.2$ and $2.2$, and their bootstrap intervals exclude zero for every system except Franka, whose nearly flat curves leave the exponent undetermined.

\noindent \textit{2) Systems Differ in the Data They Need:} The data needed before additional latent capacity becomes beneficial is substantially larger for Franka, G1, and Go2 than for the pendulums and Kinova. This behavior is consistent with their larger latent dimensions and, particularly for Go2, larger estimated normalized leverage (\cref{tab:gram_diag}). These observations support the importance of jointly scaling representation capacity and training data, although the sample requirements of end-to-end neural learning also depend on temporal dependence and representation learning.

\noindent \textit{3) Predictive Check:} To test whether the fitted laws predict unseen resource levels, we refit every curve of \cref{fig:scaling_samples,fig:scaling_encodedim} on four of its five levels and predict the fifth. The median relative prediction errors are $1.6\%$ (data) and $1.8\%$ (latent dimension) for interior levels, and $1.1\%$ and $4.3\%$ when the largest level is held out. Predictions at the smallest level are less accurate ($11.9\%$ and $15.0\%$), because the curves bend most at low budgets. The fitted laws therefore give useful predictions over the tested range, particularly toward larger budgets.

\subsubsection{Unified Scaling: Balancing Data and Model Capacity} \label{sec:unified_scaling}

\begin{figure*}[htbp]
    \centering
    \vspace{-5pt}
    \includegraphics[width=\textwidth]{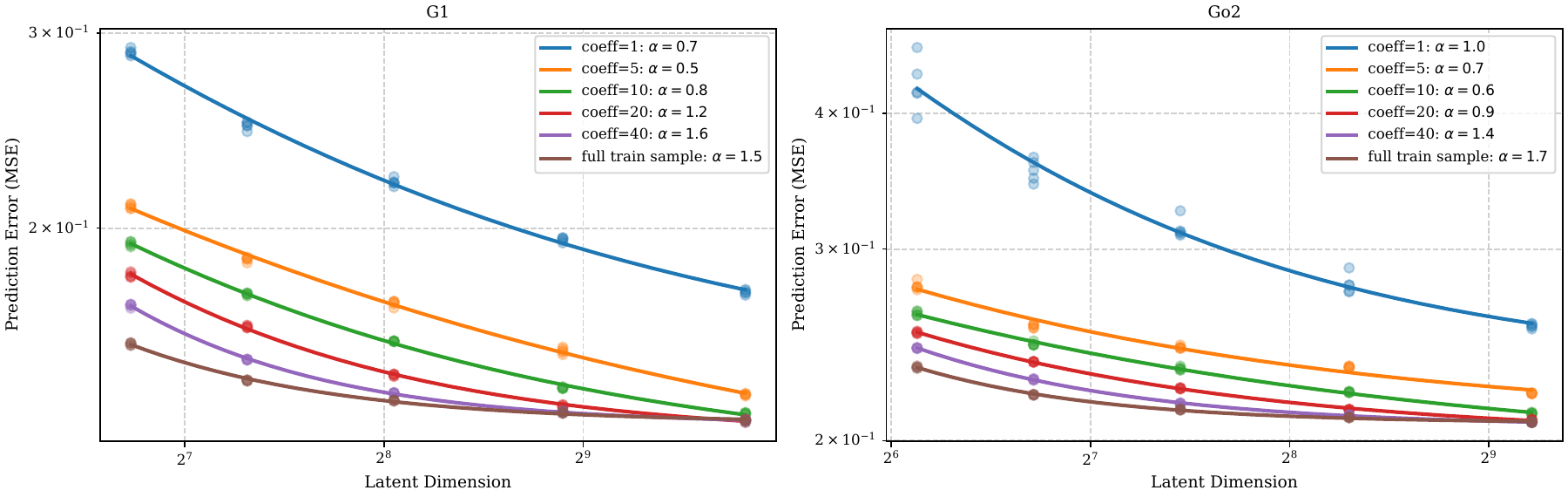}
    \vspace{-5pt}
    \caption{
    Scaling of prediction error with the total latent dimension $n$ under the coupled training-sample schedule $m=c\,n_\text{enc}\log n_\text{enc}$ ($c=\text{coeff}$), where $n_\text{enc}$ is the number of learned nonlinear features. Training budgets are capped at the 140,000 available samples. Each configuration uses both auxiliary losses and five training seeds (dots). Increasing the data-sufficiency coefficient improves prediction accuracy toward the full-dataset baseline.}
    \label{fig:mlogn_scaling}
    \vspace{-5pt}
\end{figure*}

For fixed, well-spread observable spaces, \cref{cor:scaling_law} shows that scaling the number of independent transitions proportionally to $n\log n$ keeps the sampling-error bound fixed. Motivated by this, we test a coupled schedule for end-to-end neural Koopman learning: we set the number of training samples to $m=c\,n_\text{enc}\log n_\text{enc}$, where $n_\text{enc}$ is the number of learned features and $c$ is the data-sufficiency coefficient ($\text{coeff}$ in \cref{fig:mlogn_scaling}). Requests above the 140k available samples are capped; with $c=40$ this happens only for the largest latent size.

\noindent \textit{1) Scaling Consistency:} For every coefficient, the error decreases steadily as $n$ increases, and the low-data degradation of the fixed-$m$ sweep disappears. Because $m$ and $n$ grow together, this experiment does not isolate the logarithmic factor; \cref{sec:controlled_grid} tests it directly.

\noindent \textit{2) Data Sufficiency:} Increasing $c$ from $1$ to $40$ shifts the curves down toward the full-data baseline, as the sampling level $\mathcal{O}(c^{-1/2})$ of \cref{eq:converge_kmn} predicts. For G1, $c=40$ yields $\alpha \approx 1.58$, close to the full-data exponent of this experiment ($\alpha \approx 1.50$; these runs use both auxiliary losses, so the exponents differ from \cref{tab:scaling_fits}). For $n_\text{mult}\le8$, the schedule uses at most 103k (G1) and 63k (Go2) samples and gives nearly full-data accuracy.

\noindent \textit{3) Theory and Practice:} The $n\log n$ scale follows from \cref{prop:samp_err}: the guarantee requires $m\gtrsim\kappa_n\log n\ge n\log n$, and for well-spread features the sampling term depends on $m$ and $n$ only through $m/(n\log n)$. The coefficient $40$, in contrast, is an empirical guideline for the tested architectures and data. The theory suggests that representations with larger normalized leverage may require larger data-sufficiency coefficients, although the empirical relationship also depends on representation learning and the data-collection protocol.

\subsubsection{A Direct Test of the Theory with Fixed Features}
\label{sec:controlled_grid}
The experiments above train the encoder and the transition matrices jointly, and the encoder changes with $m$, which goes beyond the setting of \cref{sec:sup_proof}. We therefore test the sampling behavior of \cref{prop:samp_err,cor:scaling_law} directly, in the controlled-regression setting of \cref{rem:controlled}. On the damped pendulum, for each of five seeds and each latent size $n\in\{4,6,10,18,34,66\}$, we train an encoder on 4{,}000 trajectories and freeze it. We then fit $(\mathbf{A},\mathbf{B})$ by minimum-norm least squares (\texttt{numpy.linalg.lstsq}, which also covers $m<n+n_u$) on $m$ independent transitions, taken one per trajectory from 8{,}000 further trajectories, and repeat each draw 40 times. Population least-squares orthogonality splits the one-step prediction MSE of a fitted predictor $\hat{\mathbf{W}}_{n,m}$ into a projection part $\mathbb{E}\|x^+-\mathbf{W}\mathbf{r}\|^2$ and a sampling part $\mathbb{E}\|(\mathbf{W}-\hat{\mathbf{W}}_{n,m})\mathbf{r}\|^2$, where $\mathbf{r}=[\bm{\Phi}(x);u]$ and $\mathbf{W}$ is the population predictor. We estimate $\mathbf{W}$ and the population Gram matrix by a large-data reference fit $\mathbf{W}_\text{ref}$ on 240{,}000 transitions of 16{,}000 separately simulated trajectories, and evaluate both parts on held-out transitions from 4{,}000 further trajectories. We also record the event $\|\tilde{\mathbf{G}}_{n,m}-\mathbf{I}\|\le\frac12$ on which the proof of \cref{prop:samp_err} rests, and the estimated leverage constant $\hat\kappa_r$ of the regressor $\mathbf{r}$, whose dimension is $N=n+n_u=n+1$.

\begin{figure*}[t]
    \centering
    \includegraphics[width=\textwidth]{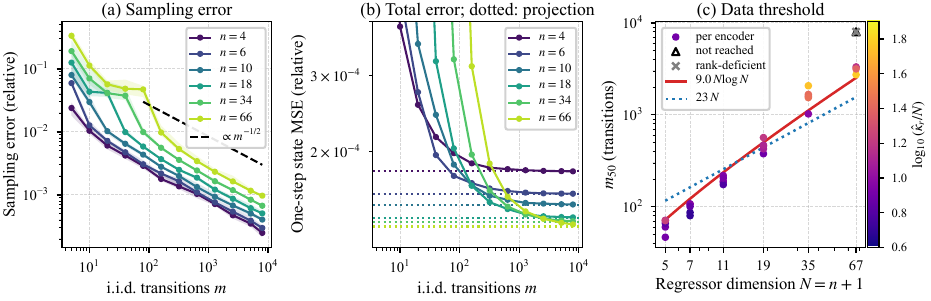}
    \caption{Direct test of \cref{prop:samp_err,cor:scaling_law} on the damped pendulum with frozen learned features and $m$ independent transitions (five seeds, 40 draws each). (a) Sampling error: root-mean-square distance of the predicted next state from the reference prediction, relative to the next state (median over seeds); dashed: $m^{-1/2}$. (b) Total one-step state MSE relative to the mean squared next state; dotted: projection MSE. (c) Number of transitions $m_{50}$ at which the event $\|\tilde{\mathbf{G}}_{n,m}-\mathbf{I}\|\le\frac12$ holds with probability $\frac12$, per encoder, colored by the estimated leverage $\hat\kappa_r/N$, with the fitted $N\log N$ and $N$ laws.}
    \label{fig:fixed_feature}
\end{figure*}

\noindent\textit{1) Sampling rate and error decomposition.} For every latent size, the sampling error decays as $m^{-0.50}$ to $m^{-0.52}$ (\cref{fig:fixed_feature}a), the rate of \cref{prop:samp_err}, and it grows with $n$ at fixed $m$. The decomposition holds to within $0.3\%$ of the total error, and the projection error, which does not depend on $m$, decreases with $n$ from $1.80\times10^{-4}$ to $1.34\times10^{-4}$ (\cref{fig:fixed_feature}b), as \cref{prop:proj_err} describes.

\noindent\textit{2) Data requirement and leverage.} We compare empirical scaling laws for the threshold $m_{50}$: linear scaling in $N$, logarithmically corrected scaling in $N\log N$, a flexible dimension-only law $N^p$, direct leverage scaling in $\hat\kappa_r\log N$, and a leverage-adjusted law $m_{50}\propto N\log N\,(\hat\kappa_r/N)^\eta$. Each law is fitted by a Gaussian likelihood for $\log m_{50}$ in which an encoder that does not reach the event within 8{,}000 transitions is right-censored. For prediction, we leave out one seed at a time, fit on the other four, and predict the uncensored thresholds of the held-out seed; Its estimated leverage $\hat\kappa_r$ is computed from separate reference data, without using the regression draws or observed sample thresholds. One encoder at $n=66$ has a Gram matrix with $\lambda_{\min}/\lambda_{\max}=2\times10^{-16}$: one feature is a linear combination of the others to numerical precision on all reference states. It violates \cref{assumption:gram}, so its threshold and leverage are not meaningful, and we exclude it from the threshold analysis. The other encoder that does not reach the event is full rank with large leverage ($\hat\kappa_r/N=75$) and is treated as right-censored.

\begin{table}[ht]
\centering
\caption{Empirical sample-threshold laws for the 29 full-rank encoders of \cref{fig:fixed_feature}c. Held-out error: leave-one-seed-out root-mean-square error of $\log m_{50}$.}
\label{tab:threshold_laws}
\small
\begin{tabular}{lcc}
\toprule
Scaling law & AIC ($\downarrow$) & Held-out error ($\downarrow$) \\
\midrule
$N$ & $59.8$ & $0.53$ \\
$N\log N$ & $27.1$ & $0.26$ \\
$N^p$ ($\hat p=1.64$) & $7.2$ & $0.22$ \\
$\hat\kappa_r\log N$ & $36.9$ & $0.44$ \\
$N\log N\,(\hat\kappa_r/N)^\eta$ ($\hat\eta=0.42$) & $\mathbf{-10.0}$ & $\mathbf{0.18}$ \\
\bottomrule
\end{tabular}
\end{table}

We compare the fitted scaling laws using the Akaike information criterion (AIC) \cite{akaike1974new}. A lower AIC indicates a better trade-off between goodness of fit and model complexity.
\Cref{tab:threshold_laws} shows that the $N\log N$ law fits substantially better than linear scaling. Incorporating the estimated leverage improves both the fit and the prediction: relative to $N\log N$, the leverage-adjusted law reduces the held-out error by $30\%$, and it also outperforms the more flexible dimension-only law $N^p$. The deviations of $m_{50}$ from the $N\log N$ law correlate with $\log(\hat\kappa_r/N)$ (rank correlation $0.84$, $p=2\times10^{-8}$, uncensored encoders). Because leverage also grows with dimension, we additionally remove the mean of each latent size: within dimensions, $\log m_{50}$ and $\log\hat\kappa_r$ still correlate ($r=0.61$, $p=6\times10^{-4}$). The model ranking is unchanged when the censored encoder is excluded (\cref{sec:ap_diag}). Estimated leverage therefore provides predictive information about sample requirements beyond latent dimension alone. The fitted dependence on leverage is sublinear ($\hat\eta\approx0.4$), and the direct law $\hat\kappa_r\log N$ fits worse than $N\log N$, whereas \cref{prop:samp_err} gives a sufficient condition that is linear in leverage. The two answer different questions: the bound is a worst-case, high-probability guarantee in terms of the population supremum, while $m_{50}$ is a 50\% threshold and $\hat\kappa_r$ an empirical estimate. The theorem thus gives a sufficient sample scale, not the exact empirical scaling law.

\noindent\textit{3) Data--capacity trade-off.} The best latent size grows with the data (\cref{fig:fixed_feature}b): $n=4$ is best in every seed up to 10 transitions, $n=18$ in four of five seeds at 320 transitions, and $n=66$ in every seed from 5{,}120 transitions on, while with little data the largest models are the worst. This is the behavior of \cref{cor:scaling_law}(ii).

\subsection{Effectiveness of the Proposed Auxiliary Losses}
\label{sec:exp_aux_loss}

\begin{figure}[htbp]
    \centering
    \vspace{-5pt}
        \includegraphics[width=1.0\linewidth]{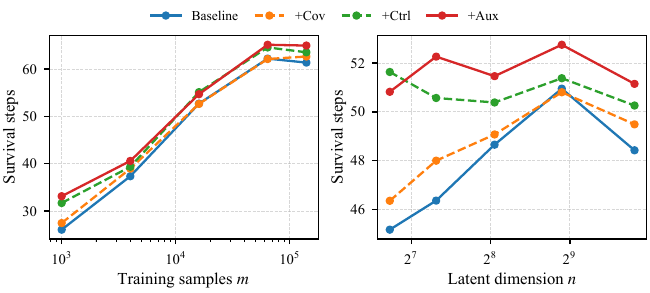}
        \vspace{-5pt}
    \caption{G1 mean survival steps on the 1{,}000 test trajectories ( a rollout fails when the mean joint-position error exceeds $0.10$ rad) versus training samples $m$ (left; mean over the five latent sizes) and latent dimension $n$ (right; mean over the five training-set sizes). Each point averages five seeds, with the loss weights and training length selected on separate validation trajectories (\cref{sec:ap_aux_tuning}). +Aux survives longer than the baseline at every point.}
    \vspace{-5pt}
    \label{fig:control_g1}
\end{figure}

\begin{figure}[htbp]
    \centering
    \vspace{-5pt}
        \includegraphics[width=1.0\linewidth]{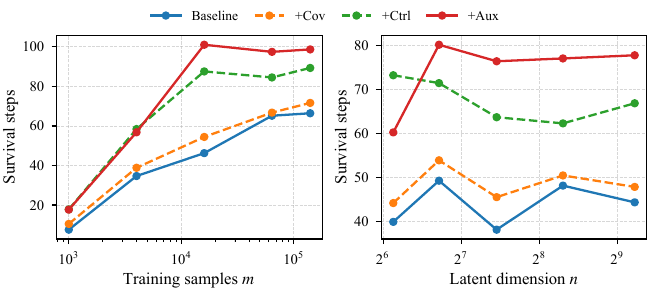}
        \vspace{-5pt}
        \caption{Go2 mean survival steps on the 1{,}000 test trajectories (failure threshold $0.16$ rad) versus training samples $m$ (left; mean over the five latent sizes) and latent dimension $n$ (right; mean over the five training-set sizes). Each point averages five seeds, with the loss weights and training length selected on separate validation trajectories (\cref{sec:ap_aux_tuning}). +Aux survive much longer than the baseline at every point.}
        \label{fig:control_go2}
        \vspace{-5pt}
\end{figure}

To assess the regularizers of \cref{sec:train_objective}, we compare four training variants: the baseline, trained with the prediction loss only; +Cov, with the covariance loss $\mathcal{L}_{\text{cov}}$; +Ctrl, with the inverse-control loss $\mathcal{L}_{\text{ctrl}}$; and +Aux, with both auxiliary losses. Prediction accuracy alone does not fully characterize control utility: models with similar prediction accuracy can exhibit substantially different closed-loop performance. At the configuration of \cref{tab:baseline_pred}, the auxiliary losses change open-loop prediction error by at most 2\%, but they have a much larger effect on closed-loop control. In the low-data regime, $\mathcal{L}_\text{ctrl}$ also reduces prediction error, by up to 10\% (\cref{sec:exp_ctrl_loss_cond}).

We deploy the learned Koopman models within the MPC framework of \cref{eq:koopman_mpc} on the Unitree G1 humanoid and the Go2 quadruped, which track 1{,}000 test reference trajectories. Training length, loss weights, and the MPC input weight are selected on separate validation trajectories (\cref{sec:ap_aux_tuning}). These tasks measure model quality: across the baseline models, the rank correlation between prediction accuracy and survival is $0.94$ on G1 and $0.84$ on Go2. On both robots, +Aux survives longer than the baseline at every point of both sweeps (\cref{fig:control_g1,fig:control_go2}). At $m=140$k it survives $64.9\pm0.7$ steps versus $61.4\pm0.8$ on G1 and $98.7\pm8.6$ versus $66.4\pm4.7$ on Go2 (mean $\pm$ standard deviation over five seeds), and it is better in all five seeds on both robots (paired $t$-test, $p=0.003$). On Go2, the baseline does not benefit from a larger latent dimension, whereas +Aux stays 20--38 steps above it at every latent size.

Of the two losses, $\mathcal{L}_\text{ctrl}$ contributes most of the gain on both robots (\cref{tab:control_summary}), which suggests that keeping the inputs recoverable matters most for underactuated, contact-rich dynamics.
Each loss has its intended effect on the learned representation (\cref{tab:reg_diag}): at $m=140$k, $\mathcal{L}_\text{ctrl}$ lowers the held-out error of reconstructing the applied input from latent transitions by a factor of 10 to 18, and the two losses together lower the condition number of the feature Gram matrix by a factor of 4 to 5 and the normalized leverage $\hat\kappa_n/n$ by 11--38\%. On the real Kinova arm, with weights selected on validation data, the inverse-control loss improves input recoverability by a factor of $3.7$ to $8$, and the combined regularizers reduce the prediction loss by $15\%$ at $m=4$k and $7\%$ at $m=140$k (\cref{sec:ap_diag}).

\begin{table}[ht]
\centering
\vspace{-5pt}
\caption{Relative improvement in survival steps over the baseline from the covariance ($\mathcal{L}_{\text{cov}}$) and inverse-control ($\mathcal{L}_{\text{ctrl}}$) losses, with the validation-selected settings of \cref{sec:ap_aux_tuning}. Each entry is the relative change at each point of a sweep in \cref{fig:control_g1,fig:control_go2}, averaged over the five points of that sweep. Positive values are improvements.}
\label{tab:control_summary}
\vspace{-5pt}
\small
\setlength{\tabcolsep}{4pt}
\begin{tabular}{llccc}
\toprule
Env & Sweep & +Cov & +Ctrl & +Aux \\ \midrule
\multirow{2}{*}{G1} & data $m$ & $2.3\%$ & $7.8\%$ & $\mathbf{10.0\%}$ \\
 & latent $n$ & $1.8\%$ & $6.3\%$ & $\mathbf{8.0\%}$ \\ \midrule
\multirow{2}{*}{Go2} & data $m$ & $15.7\%$ & $71.4\%$ & $\mathbf{82.7\%}$ \\
 & latent $n$ & $10.4\%$ & $54.9\%$ & $\mathbf{69.7\%}$ \\
\bottomrule
\end{tabular}
\vspace{-5pt}
\end{table}

\subsection{Ablation Studies}

\subsubsection{Effectiveness of Neural Koopman Operators}

\begin{table}[htbp]
\centering
\vspace{-5pt}
\caption{Open-loop prediction error of classical EDMD (polynomial basis), unstructured neural dynamics (NNDM), and neural Koopman models, at $m=140$k and the largest latent size ($n_\text{mult}=16$).}
\label{tab:baseline_pred}
\small
\setlength{\tabcolsep}{3.5pt}
\vspace{-5pt}
\begin{tabular}{ccccc}
\toprule
\multirow{2}{*}{{Environment}} & {EDMD} & {NNDM} & {Koopman} & {Koopman} \\
 & (Poly) & (MLP) & (Base NN) & {(+Aux)} \\\hline
Franka      & 0.541       & 0.385 & 0.122      & \textbf{0.121}                 \\ \hline
G1          & 0.271       & \textbf{0.116} & 0.130      & 0.132                 \\
Go2         & 4.288       & \textbf{0.163} & 0.204      & 0.208                 \\ \bottomrule
\end{tabular}
\vspace{-5pt}
\end{table}

\begin{table}[htbp]
\centering
\vspace{-5pt}
\caption{Closed-loop control: figure-8 end-effector tracking error on Franka and mean survival steps on the 1{,}000 test trajectories of G1 and Go2, at $m=140$k. Koopman columns average the five latent sizes. EDMD uses the same linear MPC and MPC setting as the Koopman models; NNDM uses random-shooting MPC with 1{,}000 samples (mean over three seeds).}
\label{tab:baseline_control}
\small
\setlength{\tabcolsep}{3.5pt}
\vspace{-5pt}
\begin{tabular}{lccccc}
\toprule
\multirow{2}{*}{{Metric}} & \multirow{2}{*}{{Env}} & {EDMD} & {NNDM} & {Koopman} & {Koopman} \\
 & & (Poly) & (MLP) & (Base NN) & {(+Aux)} \\
\midrule
Track.\ err.\ ($\downarrow$) & Franka & 0.264 & 0.254 & 0.233 & \textbf{0.205} \\
\midrule
\multirow{2}{*}{Survival ($\uparrow$)} & G1 & 42.9 & 5.4 & 61.4 & \textbf{64.9} \\
 & Go2 & 95.1 & 1.6 & 66.4 & \textbf{98.7} \\
\bottomrule
\end{tabular}
\vspace{-5pt}
\end{table}

We compare our models with two baselines. EDMD with a polynomial basis (EDMD-Poly) \cite{han2023utility} is structured but not neural. The unstructured neural dynamics model (NNDM) \cite{nagabandi2017neuralnetworkdynamicsmodelbased} is an MLP $x_{t+1} = f_\theta(x_t, u_t)$ with the same number of parameters as our Koopman model; it is neural but unstructured, so its MPC problem is nonconvex, and we solve it by random shooting with 1{,}000 samples. Koopman (Base NN) is the baseline, trained with $\mathcal{L}_{\text{pred}}$ only, and Koopman (+Aux) is the variant with both auxiliary losses (\cref{sec:exp_aux_loss}). EDMD and the Koopman models use the linear MPC of \cref{sec:koopman_mpc_form}, which solves the unconstrained quadratic problem in closed form and clips the first input to the actuator limits. \Cref{tab:baseline_pred} evaluates the largest latent size, whereas the G1 and Go2 survival in \cref{tab:baseline_control} averages the five latent sizes with the validation-selected settings, as in \cref{fig:control_g1,fig:control_go2}; the two tables therefore need not agree model by model.

\Cref{tab:baseline_pred} shows that learned representations predict much better than fixed polynomial features; across seeds, the standard deviation of the Koopman errors is at most $0.004$. Although NNDM achieves the lowest prediction error on G1 and Go2, the linear latent structure of Koopman models enables efficient quadratic optimization for control.

\Cref{tab:baseline_control} shows that prediction accuracy alone does not characterize control performance. Despite its accurate predictions, NNDM performs poorly with random-shooting MPC. Random shooting with 1{,}000 samples also performs poorly with the Koopman predictors, indicating that planning is a major bottleneck in these high-dimensional control settings; the linear structure makes a reliable controller cheap. On G1, both Koopman models survive longer than both baselines, and on both robots the proposed regularizers give further gains. EDMD is the strongest baseline on Go2: it survives longer than the Koopman baseline and nearly as long as +Aux, although its open-loop error is about 20 times higher. EDMD fits $\mathbf{B}$ directly by least squares, whereas the baseline Koopman model learns $\mathbf{B}$ only through the multi-step prediction loss; this is consistent with the gain from $\mathcal{L}_\text{ctrl}$, which also keeps the input effect explicit. On Franka, +Aux reduces the tracking error by $12\%$ relative to the Koopman baseline ($0.205$ versus $0.233$, better in all five seeds) and tracks more accurately than both baselines (\cref{sec:ap_franka_control}). These results highlight the importance of jointly considering prediction accuracy, model structure, and downstream control objectives.

\subsubsection{Relationship Between Covariance Regularization and Prediction Error}
\label{sec:exp_cov_loss_cond}

In \cref{sec:sup_proof}, the sampling bound depends on the features only through their span, via $\kappa_n$, so correlation between features does not change it. Feature geometry matters for fitting instead, through the condition number of the Gram matrix. We measure the condition number $\operatorname{cond}(\boldsymbol{\Sigma}_\Psi)$ of the covariance of the learned features $\Psi_\theta(x)$ on test states, which is the matrix that $\mathcal{L}_\text{cov}$ regularizes. On G1 and Go2, $\mathcal{L}_\text{cov}$ lowers it by a factor of $1.5$ to $5$ at every sample size, while the prediction error changes by at most about $2\%$ (\cref{fig:cov_loss_cond}, \cref{sec:ap_diag}). Covariance regularization thus substantially improves feature conditioning, while its effect on prediction accuracy is modest: numerical conditioning and prediction accuracy characterize different aspects of a learned representation. Because $\mathcal{L}_\text{cov}$ also changes the learned subspace, this comparison does not isolate the effect of conditioning. The benefit of $\mathcal{L}_\text{cov}$ appears in closed loop instead, for example on Go2 (\cref{tab:control_summary}).

\paragraph{Conditioning as the latent dimension grows.}
\Cref{tab:cond_vs_n} reports $\operatorname{cond}(\boldsymbol{\Sigma}_\Psi)$ across latent dimensions at the largest dataset. From $n_{\text{mult}}=1$ to $16$, it grows by one to three orders of magnitude; $\mathcal{L}_\text{cov}$ lowers it by a factor of $2$ to $30$ at every dimension, with the largest reductions at small $n$, but it does not stop the growth. The full, uncentered Gram matrix $\hat{\mathbf{G}}_n$ of $\bm{\Phi}=[x;\Psi_\theta(x)]$ on held-out states answers the question directly (\cref{tab:gram_diag}): for the G1 and Go2 models it has full numerical rank, its smallest eigenvalue decreases with $n$ (G1: from $8.5\times10^{-3}$ to $1.9\times10^{-3}$; Go2: from $3.6\times10^{-3}$ to $1.7\times10^{-3}$), and its condition number grows by a factor of 7 to 40. Conditioning thus deteriorates as the latent dimension increases, even with covariance regularization. The leverage constant behaves differently: it does not depend on the feature basis, and the estimated $\hat\kappa_n/n$ stays nearly constant over the tested dimensions (about 8 for G1 and 43--59 for Go2). The statistical sample requirement of an observable space and the numerical conditioning of its basis therefore need not scale in the same way. Prediction error still decreases with $n$ at this dataset size (e.g., G1: $0.158\to0.142$; Go2: $0.232\to0.215$, for runs without the inverse-control loss).

\begin{table*}[t]
\centering
\caption{Condition number $\operatorname{cond}(\boldsymbol{\Sigma}_\Psi)$ of the learned-feature covariance versus the latent multiplier $n_\text{mult}$ at $m=140$k, for models trained without / with $\mathcal{L}_\text{cov}$ and without the inverse-control loss. Entries are medians over seeds.
}
\label{tab:cond_vs_n}
\small
\begin{tabular}{lccccc}
\toprule
System & $n_{\text{mult}}=1$ & 2 & 4 & 8 & 16 \\
\midrule
Franka & 124 / 12 & 2321 / 184 & 4043 / 493 & 7035 / 1681 & 22660 / 9407 \\
G1 & 172 / 11 & 891 / 77 & 1182 / 140 & 1284 / 288 & 4603 / 1825 \\
Go2 & 246 / 8 & 1462 / 58 & 2607 / 173 & 2927 / 290 & 5390 / 1444 \\
\bottomrule
\end{tabular}
\end{table*}

\subsubsection{Relationship Between Inverse Control Loss and Prediction Error}
\label{sec:exp_ctrl_loss_cond}

The inverse-control loss acts through the control matrix $\mathbf{B}$. It reduces prediction error in the low-data regime: at $m=10^3$ and $n_\text{mult}=16$, the error drops from $0.364$ to $0.340$ on G1 and from $0.510$ to $0.459$ on Go2. With more data, the errors of the two variants converge. Its effect on $\operatorname{cond}(\mathbf{B}^\top\mathbf{B})$ is not consistent: at $m=10^3$ it raises the condition number (G1: $2.7\to3.2$), and at $m=1.4\times10^5$ it lowers it (G1: $12.0\to6.9$; Go2: $6.1\to4.4$). Improvements from inverse-control regularization are therefore not consistently associated with better conditioning of $\mathbf{B}^\top\mathbf{B}$, but they are accompanied by improved input recoverability (\cref{tab:reg_diag}).

\section{Conclusion}

This paper investigated how training data and latent representation capacity jointly determine the prediction and control performance of neural Koopman models. By decomposing the Koopman approximation error into sampling and projection components, we established a leverage-dependent finite-sample bound, also for the linear-input predictor used in control, and derived a data--capacity scaling rule under explicit approximation assumptions. The statistical sample requirement of an observable space thus depends on its geometry as well as its dimension. A fixed-feature regression study supports this picture: the sampling error decays at the predicted rate, and the measured leverage constant improves the prediction of the required data beyond what the dimension alone explains. Estimated leverage thus offers a practical diagnostic of the sample requirements of learned observable spaces for data--capacity allocation. Across six mechanical and robotic systems, increasing latent capacity improves prediction when supported by sufficient data but can degrade it in data-limited regimes. Covariance and inverse-control regularization improve downstream control, reducing Franka tracking error by $12\%$ and increasing tracking survival by $6\%$ on G1 and $49\%$ on Go2 at the largest dataset. In practice, the latent dimension should grow with the data (about $40\,n_\text{enc}\log n_\text{enc}$ training samples gave nearly full-data accuracy on G1 and Go2), and models should be evaluated in closed loop, not only by their prediction error.

Extending the guarantees to encoders trained on the same data and to temporally dependent trajectories with hybrid contact dynamics, and validating the locomotion results on hardware, remain directions for future work.

\section*{Acknowledgment}

The authors thank Hongyi Chen for insightful discussions on Koopman operator theory and Zhongqi Wei for collecting the Kinova dataset.   
This work was supported by the National Science Foundation (NSF) under Grant 2144489.

\bibliographystyle{IEEEtran}
\bibliography{main} 

\clearpage
\appendix
\section{Supplementary Material}
\label{sec:appendix}

\subsection{Experimental Setup} \label{sec:ap_exp_setup}
\Cref{tab:notation} summarizes the main notation.

\begin{table}[ht]
\centering
\caption{Main notation.}
\label{tab:notation}
\small
\begin{tabular}{@{}lp{0.64\linewidth}@{}}
\toprule
Symbol & Meaning \\
\midrule
$x$, $u$; $n_x$, $n_u$ & physical state and control input; their dimensions \\
$\bm{\Phi}(x)=[x;\Psi_\theta(x)]$ & latent state, with $n_\text{enc}$ learned features $\Psi_\theta(x)$ \\
$n$, $n_\text{mult}$ & latent dimension $n=n_x+n_\text{enc}$; multiplier $n_\text{mult}=n_\text{enc}/n_x$ \\
$m$ & number of training samples: one-step transitions (\cref{sec:sup_proof,sec:controlled_grid}), trajectory segments of 15 transitions (other experiments) \\
$\kappa_n$, $\kappa_r$ & leverage constants of the latent space and of the regressor $[\bm{\Phi}(x);u]$ \\
$\beta$; $n^\star$ & approximation rate (\cref{assumption:approx}); bound-optimal latent dimension \\
$\alpha_m$, $\alpha_n$ & fitted data and capacity exponents of \cref{eq:practical_scaling_law} \\
$\log$ & natural logarithm \\
\bottomrule
\end{tabular}
\end{table}

We evaluate neural Koopman learning on the six systems of \cref{tab:environments}. The damped pendulum (state: angle and angular velocity; input: one torque) and the double pendulum (state: two angles and two angular velocities; inputs: torques at both joints) are simulated by numerical integration with $\Delta t=0.02$\,s and $0.01$\,s, from uniformly sampled initial states and randomized inputs. The simulated Franka Panda (PyBullet) has 7 joint positions and velocities as state and commanded joint velocities as input; trajectories start near a home configuration and apply random velocity commands. The real Kinova Gen3 has the same state and input types. Its data come from a two-stage kinematic trajectory optimization in Drake: a feasible trajectory between random configurations, then a refinement that enforces collision avoidance and joint limits. The G1 humanoid and Go2 quadruped are simulated in Isaac Sim. Their state contains the base orientation (quaternion), base linear and angular velocities, and joint positions and velocities, and their input is the desired joint positions of embedded PD controllers. Their data combine rollouts of a pretrained RL expert with rollouts of Koopman-based MPC that tracks expert trajectories, which covers states that the controller visits in closed loop.

\paragraph{Kinova data split}
The Kinova dataset has five continuous recordings. We cut segments of 16 consecutive time steps (15 transitions), with a new segment every four steps, so neighboring segments overlap. We order the segments in time across the recordings. The last 20k segments form the test set, the 20k before them form the validation set, and the training segments are drawn at random from the remaining data. All validation and test segments lie in the last recording, which was made on a later day. Segments that cross a split boundary are dropped, so no time step is shared between splits.

\subsection{Implementation Details}
\label{sec:ap_impl}
\Cref{tab:impl} lists the settings of the scaling sweeps. The other studies use the same settings unless stated otherwise. The configuration files and seeds are in the code repository.

\begin{table}[htbp]
\centering
\caption{Training and control settings.}
\label{tab:impl}
\small
\begin{tabular}{@{}lp{0.68\linewidth}@{}}
\toprule
Setting & Value \\
\midrule
Encoder $\Psi_\theta$ & MLP, three hidden layers of width 256, ReLU, residual connections \\
Optimizer & Adam, learning rate $10^{-3}$ (decayed by $0.9$ every 1{,}000 steps), gradient clipping at norm $0.1$ \\
Batch; steps & 64 (128 for Franka); 60{,}000 (20{,}000 for G1 and Go2) \\
Rollout $T$; $\rho$ & 15; 0.8 (0.99 for Franka, G1, and Go2) \\
Normalization & states and inputs standardized for Franka, G1, Go2, and the Kinova diagnostics of \cref{tab:reg_diag} \\
Data split & 140k / 20k / 20k samples (train / validation / test); smaller training sets are the first $m$ training samples \\
Loss weights & $w_\text{cov}=w_\text{ctrl}=1$ unless selected on validation data (\cref{sec:ap_aux_tuning}) \\
Seeds; selection & five fixed seeds (listed in the repository); checkpoint with the lowest validation $\mathcal{L}_\text{pred}$ \\
MPC & horizon $H=15$, $Q=\mathbf{I}$, $R=r\mathbf{I}$ ($r=0$ unless selected); closed-form solution, first input clipped to the actuator limits \\
\bottomrule
\end{tabular}
\end{table}

\subsection{Proof of \texorpdfstring{\cref{rem:controlled}}{Remark 1}}
\label{sec:ap_proof_remark}
Write $s=(x,u)\sim\nu$, $\mathbf{r}=[\bm{\Phi}(x);u]\in\mathbb{R}^{n+n_u}$, and $\mathbf{y}=\bm{\Phi}(x^+)\in\mathbb{R}^n$. Least squares on $m$ samples and its population counterpart are
\begin{align}
\hat{\mathbf{W}}_m&=\mathbf{C}_m\mathbf{R}_m^{-1},\qquad \mathbf{W}=\mathbf{C}\mathbf{R}^{-1},\nonumber\\
\mathbf{R}_m&=\frac1m\sum_{i=1}^m\mathbf{r}_i\mathbf{r}_i^\top,\qquad \mathbf{C}_m=\frac1m\sum_{i=1}^m\mathbf{y}_i\mathbf{r}_i^\top, \nonumber\\
\mathbf{R}&=\mathbb{E}_\nu[\mathbf{r}\mathbf{r}^\top],\qquad \mathbf{C}=\mathbb{E}_\nu[\mathbf{y}\mathbf{r}^\top]. \nonumber
\end{align}

\textit{Step 1 (whitening).} Let $\tilde{\mathbf{r}}=\mathbf{R}^{-1/2}\mathbf{r}$ and $\tilde{\mathbf{y}}=\mathbf{G}_n^{-1/2}\mathbf{y}$, with $\mathbf{G}_n=\mathbb{E}_\mu[\bm{\Phi}\bm{\Phi}^\top]$. By the definitions of the two leverage constants,
\begin{align}
\mathbb{E}_\nu[\tilde{\mathbf{r}}\tilde{\mathbf{r}}^\top]=\mathbf{I},\quad \|\tilde{\mathbf{r}}\|_2^2\le\kappa_r,\quad \|\tilde{\mathbf{y}}\|_2^2\le\kappa_n\ \ \text{(all states)}. \nonumber
\end{align}
For $\psi=\mathbf{c}^\top\bm{\Phi}\in\mathcal{H}_n$, the fitted and population least-squares one-step values are $\mathbf{c}^\top\hat{\mathbf{W}}_m\mathbf{r}$ and $\mathbf{c}^\top\mathbf{W}\mathbf{r}$. Let $\tilde{\mathbf{c}}=\mathbf{G}_n^{1/2}\mathbf{c}$, $\tilde{\mathbf{W}}=\mathbf{G}_n^{-1/2}\mathbf{W}\mathbf{R}^{1/2}$, and $\tilde{\mathbf{W}}_m=\mathbf{G}_n^{-1/2}\hat{\mathbf{W}}_m\mathbf{R}^{1/2}$, so that $\mathbf{c}^\top\mathbf{W}\mathbf{r}=\tilde{\mathbf{c}}^\top\tilde{\mathbf{W}}\tilde{\mathbf{r}}$ and $\|\tilde{\mathbf{c}}\|_2=\|\psi\|_{\mathcal{H}}$. Because $\mathbb{E}_\nu[\tilde{\mathbf{r}}\tilde{\mathbf{r}}^\top]=\mathbf{I}$,
\begin{align}
\big\|\mathbf{c}^\top(\hat{\mathbf{W}}_m-\mathbf{W})\mathbf{r}\big\|_{L_2(\nu)}&=\big\|(\tilde{\mathbf{W}}_m-\tilde{\mathbf{W}})^\top\tilde{\mathbf{c}}\big\|_2\nonumber\\
&\le\|\tilde{\mathbf{W}}_m-\tilde{\mathbf{W}}\|_\sigma\,\|\psi\|_{\mathcal{H}}. \nonumber
\end{align}
It therefore suffices to bound $\|\tilde{\mathbf{W}}_m-\tilde{\mathbf{W}}\|_\sigma$.

\textit{Step 2 (decomposition).} Let $\tilde{\mathbf{R}}_m=\mathbf{R}^{-1/2}\mathbf{R}_m\mathbf{R}^{-1/2}$, $\tilde{\mathbf{C}}_m=\mathbf{G}_n^{-1/2}\mathbf{C}_m\mathbf{R}^{-1/2}$, and $\tilde{\mathbf{C}}=\mathbf{G}_n^{-1/2}\mathbf{C}\mathbf{R}^{-1/2}$. Then $\tilde{\mathbf{W}}_m=\tilde{\mathbf{C}}_m\tilde{\mathbf{R}}_m^{-1}$, $\tilde{\mathbf{W}}=\tilde{\mathbf{C}}$, and
\begin{align}
\tilde{\mathbf{W}}_m-\tilde{\mathbf{W}}=(\tilde{\mathbf{C}}_m-\tilde{\mathbf{C}})\tilde{\mathbf{R}}_m^{-1}+\tilde{\mathbf{C}}(\tilde{\mathbf{R}}_m^{-1}-\mathbf{I}). \nonumber
\end{align}
On the event $\|\tilde{\mathbf{R}}_m-\mathbf{I}\|_\sigma\le\frac12$, the argument of Step 2 of the proof of \cref{prop:samp_err} gives $\|\tilde{\mathbf{R}}_m^{-1}\|_\sigma\le2$ and
\begin{align}
\|\tilde{\mathbf{W}}_m-\tilde{\mathbf{W}}\|_\sigma\le2\|\tilde{\mathbf{C}}_m-\tilde{\mathbf{C}}\|_\sigma+2\|\tilde{\mathbf{C}}\|_\sigma\|\tilde{\mathbf{R}}_m-\mathbf{I}\|_\sigma. \nonumber
\end{align}
For unit vectors $\mathbf{a}$ and $\mathbf{b}$, the Cauchy--Schwarz inequality gives
\begin{align}
\mathbf{a}^\top\tilde{\mathbf{C}}\mathbf{b}&=\mathbb{E}\big[(\mathbf{a}^\top\tilde{\mathbf{y}})(\mathbf{b}^\top\tilde{\mathbf{r}})\big]\nonumber\\
&\le\big(\mathbb{E}[(\mathbf{a}^\top\tilde{\mathbf{y}})^2]\big)^{1/2}\big(\mathbb{E}[(\mathbf{b}^\top\tilde{\mathbf{r}})^2]\big)^{1/2}. \nonumber
\end{align}
Here $\mathbb{E}[(\mathbf{b}^\top\tilde{\mathbf{r}})^2]=1$, and the bound $\mu^+\le M^2\mu$ on the law of $x^+$ gives
\begin{align}
\mathbb{E}[(\mathbf{a}^\top\tilde{\mathbf{y}})^2]&=\mathbb{E}\big[(\mathbf{a}^\top\mathbf{G}_n^{-1/2}\bm{\Phi}(x^+))^2\big]\nonumber\\
&\le M^2\,\mathbb{E}_\mu\big[(\mathbf{a}^\top\mathbf{G}_n^{-1/2}\bm{\Phi})^2\big]=M^2. \nonumber
\end{align}
Hence $\|\tilde{\mathbf{C}}\|_\sigma\le M$ and $\mathbb{E}[\tilde{\mathbf{y}}\tilde{\mathbf{y}}^\top]\preceq M^2\mathbf{I}$.

\textit{Step 3 (concentration).} For the Gram matrix, $\mathbf{Y}_i=\frac1m(\tilde{\mathbf{r}}_i\tilde{\mathbf{r}}_i^\top-\mathbf{I})$ has $L=\kappa_r/m$ and $v^2\le\kappa_r/m$, exactly as in the proof of \cref{prop:samp_err}. For the cross moment, $\mathbf{Y}_i=\frac1m(\tilde{\mathbf{y}}_i\tilde{\mathbf{r}}_i^\top-\tilde{\mathbf{C}})$. Since $\|\tilde{\mathbf{y}}\tilde{\mathbf{r}}^\top\|_\sigma=\|\tilde{\mathbf{y}}\|_2\|\tilde{\mathbf{r}}\|_2\le\sqrt{\kappa_n\kappa_r}\le\bar\kappa$ and, by Jensen's inequality, $\|\tilde{\mathbf{C}}\|_\sigma\le\bar\kappa$, we have $L=2\bar\kappa/m$. The variance terms satisfy
\begin{align}
\sum_i\mathbb{E}[\mathbf{Y}_i\mathbf{Y}_i^\top]&\preceq\frac1m\mathbb{E}\big[\|\tilde{\mathbf{r}}\|_2^2\,\tilde{\mathbf{y}}\tilde{\mathbf{y}}^\top\big]\preceq\frac{\kappa_rM^2}{m}\mathbf{I},\nonumber\\
\sum_i\mathbb{E}[\mathbf{Y}_i^\top\mathbf{Y}_i]&\preceq\frac1m\mathbb{E}\big[\|\tilde{\mathbf{y}}\|_2^2\,\tilde{\mathbf{r}}\tilde{\mathbf{r}}^\top\big]\preceq\frac{\kappa_n}{m}\mathbf{I}, \nonumber
\end{align}
so $v^2\le\bar\kappa\bar M^2/m$. Both matrices have at most $n+n_u$ rows and columns, so $d_1+d_2\le2(n+n_u)$, and \cref{eq:bernstein_inverted} holds for each with $\ell=\log(2(n+n_u)/\delta)$ and probability at least $1-\delta$.

\textit{Step 4 (conclusion).} Under $m\ge12\bar\kappa\ell$, the Gram bound gives $\|\tilde{\mathbf{R}}_m-\mathbf{I}\|_\sigma\le\sqrt{1/6}+1/18<\frac12$, so the event of Step 2 holds and $\mathbf{R}_m$ is invertible. By a union bound, both concentration bounds hold with probability at least $1-2\delta$. Substituting them into Step 2 gives the bound of \cref{eq:thoery_est_err3} with $\kappa_n$ replaced by $\bar\kappa$, and Step 1 converts it into the stated bound on the predicted values. \hfill$\blacksquare$

\subsection{Representation Diagnostics}
\label{sec:ap_diag}

\Cref{tab:gram_diag} reports the Gram matrix $\hat{\mathbf{G}}_n$ of $\bm{\Phi}=[x;\Psi_\theta(x)]$ and the normalized leverage $\hat\kappa_n/n$ for the pendulum encoders of \cref{sec:controlled_grid} and for the G1 and Go2 models of the tuned closed-loop study. The smallest eigenvalue decreases with the latent dimension on every system, by a factor of about 2{,}500 on the pendulum, 4.5 on G1, and 2 on Go2, and the condition number reaches $8\times10^7$, $7\times10^4$, and $2.4\times10^4$. The normalized leverage, in contrast, stays in a limited range: $7.6$--$28.8$ for the pendulum, $7.2$--$9.0$ for G1, and $29$--$59$ for Go2. This illustrates the distinction between the numerical conditioning of a feature basis and the statistical properties of the observable space it spans. One pendulum encoder at $n=66$ (seed 20965) has $\lambda_{\min}/\lambda_{\max}=2\times10^{-16}$ (numerical rank $65$): one feature is a linear combination of the others to numerical precision on all 240{,}000 reference states. The other four encoders at $n=66$ are ill-conditioned but full rank ($\lambda_{\min}/\lambda_{\max}$ between $7\times10^{-9}$ and $2\times10^{-7}$). Because the singular encoder violates \cref{assumption:gram}, \cref{sec:controlled_grid} excludes it from the threshold analysis. As a sensitivity check, the ranking of \cref{tab:threshold_laws} is unchanged when it is included (AIC $69.5$, $40.3$, $17.8$, $42.7$, $15.2$; held-out errors $0.54$, $0.27$, $0.24$, $0.45$, $0.20$; $\hat\eta=0.46$) and when the censored encoder is also excluded (AIC $47.0$, $6.9$, $-10.7$, $36.8$, $-24.0$; held-out errors $0.53$, $0.26$, $0.20$, $0.44$, $0.17$; $\hat\eta=0.34$), with the laws in the order of \cref{tab:threshold_laws}.

We compute $\hat{\mathbf{G}}_n$ on one set of states (the reference pool for the pendulum; half of 60{,}000 normalized held-out states for the robots) and take $\hat\kappa_n=\max_s\bm{\Phi}(s)^\top\hat{\mathbf{G}}_n^{-1}\bm{\Phi}(s)$ over a disjoint set. Because $\hat{\mathbf{G}}_n$ is itself estimated and the maximum is taken over finitely many states, $\hat\kappa_n$ is neither an upper nor a guaranteed lower bound on $\kappa_n$, and it depends on rare states: for individual models, another sample of states changes $\hat\kappa_n$ by up to about $50\%$. We therefore treat $\hat\kappa_n$ as an empirical diagnostic, not a certified bound.

\begin{table*}[t]
\centering
\caption{Gram matrix $\hat{\mathbf{G}}_n$ and normalized leverage $\hat\kappa_n/n$ of the learned features (medians over five seeds). Pendulum: frozen encoders of \cref{sec:controlled_grid}. G1 and Go2: models trained on $m=140$k samples, without / with $\mathcal{L}_\text{cov}$.}
\label{tab:gram_diag}
\small
\setlength{\tabcolsep}{4pt}
\begin{tabular}{lccccc}
\toprule
System & $n$ & $\lambda_{\min}$ & $\lambda_{\max}$ & $\operatorname{cond}$ & $\hat\kappa_n/n$ \\
\midrule
Pendulum & 4 & $2.2\times10^{-2}$ & 15.9 & $1.1\times10^{3}$ & 9.6 \\
Pendulum & 6 & $2.9\times10^{-3}$ & 212.8 & $9.4\times10^{4}$ & 7.6 \\
Pendulum & 10 & $7.6\times10^{-4}$ & 401.0 & $6.5\times10^{5}$ & 9.5 \\
Pendulum & 18 & $3.4\times10^{-4}$ & 1273.3 & $5.7\times10^{6}$ & 13.1 \\
Pendulum & 34 & $1.0\times10^{-4}$ & 10.3 & $1.4\times10^{5}$ & 27.1 \\
Pendulum & 66 & $8.5\times10^{-6}$ & 281.0 & $8.4\times10^{7}$ & 28.8 \\
\midrule
G1 & 106 & $8.5\times10^{-3}\,/\,4.9\times10^{-3}$ & $15.5\,/\,6.9$ & $1.8\times10^{3}\,/\,1.4\times10^{3}$ & $7.8\,/\,7.5$ \\
G1 & 159 & $5.2\times10^{-3}\,/\,3.4\times10^{-3}$ & $28.1\,/\,7.0$ & $6.2\times10^{3}\,/\,2.1\times10^{3}$ & $8.5\,/\,8.4$ \\
G1 & 265 & $3.6\times10^{-3}\,/\,2.9\times10^{-3}$ & $46.0\,/\,9.0$ & $1.3\times10^{4}\,/\,3.1\times10^{3}$ & $8.0\,/\,9.0$ \\
G1 & 477 & $3.3\times10^{-3}\,/\,2.6\times10^{-3}$ & $94.3\,/\,15.5$ & $2.9\times10^{4}\,/\,5.9\times10^{3}$ & $8.5\,/\,7.2$ \\
G1 & 901 & $1.9\times10^{-3}\,/\,5.3\times10^{-4}$ & $114.7\,/\,25.9$ & $6.8\times10^{4}\,/\,5.1\times10^{4}$ & $8.0\,/\,8.0$ \\
\midrule
Go2 & 70 & $3.6\times10^{-3}\,/\,2.1\times10^{-3}$ & $11.7\,/\,7.6$ & $3.3\times10^{3}\,/\,3.5\times10^{3}$ & $57.7\,/\,56.1$ \\
Go2 & 105 & $3.6\times10^{-3}\,/\,2.3\times10^{-3}$ & $15.6\,/\,8.0$ & $4.3\times10^{3}\,/\,3.6\times10^{3}$ & $59.1\,/\,44.6$ \\
Go2 & 175 & $3.2\times10^{-3}\,/\,2.0\times10^{-3}$ & $19.5\,/\,8.7$ & $6.2\times10^{3}\,/\,4.6\times10^{3}$ & $56.8\,/\,36.6$ \\
Go2 & 315 & $2.4\times10^{-3}\,/\,2.3\times10^{-3}$ & $27.3\,/\,10.8$ & $1.1\times10^{4}\,/\,4.8\times10^{3}$ & $57.0\,/\,33.3$ \\
Go2 & 595 & $1.7\times10^{-3}\,/\,1.4\times10^{-3}$ & $43.3\,/\,15.3$ & $2.4\times10^{4}\,/\,1.1\times10^{4}$ & $43.4\,/\,28.8$ \\
\bottomrule
\end{tabular}
\end{table*}

\begin{figure}[t]
\centering
\vspace{-5pt}
\begin{subfigure}[b]{0.49\columnwidth}
  \centering
  \includegraphics[width=\linewidth]{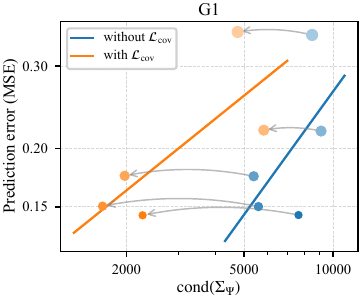}
\end{subfigure}\hfill
\begin{subfigure}[b]{0.49\columnwidth}
  \centering
  \includegraphics[width=\linewidth]{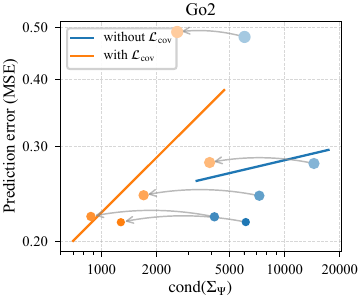}
\end{subfigure}
\vspace{-5pt}
\caption{ Condition number $\operatorname{cond}(\boldsymbol{\Sigma}_\Psi)$ of the learned-feature covariance versus prediction error. Marker size and opacity denote sample size $m$ (smaller/darker indicates larger $m$). Lines show fitted trends with (orange) and without (blue) covariance regularization. Grey arrows connect identical sample sizes ($n_{\text{mult}}=16$). The covariance loss reduces $\operatorname{cond}(\boldsymbol{\Sigma}_\Psi)$ by a factor of $1.5$--$5$ at every sample size (leftward shift) while leaving the prediction error essentially unchanged.
}
\vspace{-5pt}
\label{fig:cov_loss_cond}
\end{figure}

\Cref{fig:cov_loss_cond} relates the condition number of the learned-feature covariance $\boldsymbol{\Sigma}_\Psi$, the matrix that $\mathcal{L}_\text{cov}$ regularizes, to prediction error at $n_\text{mult}=16$. The covariance loss shifts every model to better conditioning, while the prediction errors at matched sample sizes differ by at most about $2\%$ (e.g., Go2 at $m=140$k: $0.218$ without and $0.217$ with $\mathcal{L}_\text{cov}$).

\Cref{tab:reg_diag} examines the individual and combined effects of the auxiliary losses at $n_\text{mult}=16$. The input-reconstruction error is $\mathbb{E}\|\hat u-u\|^2/\mathbb{E}\|u\|^2$ on held-out transitions, with $\hat u$ from \cref{eq:inv_control}. On G1 and Go2, $\mathcal{L}_\text{ctrl}$ lowers this error by a factor of 6 to 18, and $\mathcal{L}_\text{cov}$ lowers the Gram condition number; the two losses together give the best-conditioned features. These effects come with changes of at most $4\%$ in prediction loss. On Go2, survival follows the input-reconstruction error; on G1, the variants differ less in survival.

The Kinova rows use the real-robot data with standardized states, so their $\mathcal{L}_\text{pred}$ is not comparable with the unnormalized errors of \cref{fig:scaling_samples}. For each variant, we train $m\in\{4,140\}$k with five seeds and select the loss weights by validation prediction loss from $w\in\{0.01,0.1,1\}$, and for +Aux also from pairs with $w_\text{cov}\in\{0.001,0.01,0.1\}$ and $w_\text{ctrl}\in\{0.001,0.01,0.1,1\}$. The selected weights are $w_\text{cov}=0.01$ ($m=4$k) or $0.1$ ($m=140$k) for +Cov, $w_\text{ctrl}=0.01$ for +Ctrl, and $(w_\text{cov},w_\text{ctrl})=(0.01,0.001)$ for +Aux. The inverse-control loss lowers the input-reconstruction error by a factor of $3.7$ to $8$, and the two losses lower the Gram condition number by a factor of $2$ to $3.5$. The combined regularizers also reduce the prediction loss, by $15\%$ at $m=4$k and $7\%$ at $m=140$k. Larger inverse-control weights ($w_\text{ctrl}\ge0.1$) raise the prediction loss on this dataset.

\begin{table*}[t]
\centering
\caption{Diagnostics of the four training variants at $n_\text{mult}=16$: median normalized leverage $\hat\kappa_n/n$ and Gram condition number, and means of the held-out input-reconstruction error, test prediction loss $\mathcal{L}_\text{pred}$, and survival steps over five seeds. G1 and Go2: tuned closed-loop study. Kinova: real-robot data with standardized states, loss weights selected by validation loss; there is no closed-loop test.}
\label{tab:reg_diag}
\small
\setlength{\tabcolsep}{3.5pt}
\begin{tabular}{lllccccc}
\toprule
System & $m$ & Variant & $\hat\kappa_n/n$ & $\operatorname{cond}$ & Input rec. & $\mathcal{L}_\text{pred}$ & Survival \\
\midrule
G1 & 4k & Baseline & $8.5$ & $4.3{\times}10^{4}$ & $0.534$ & $0.192$ & $41.1$ \\
G1 & 4k & +Cov & $7.6$ & $2.1{\times}10^{4}$ & $0.463$ & $0.193$ & $40.3$ \\
G1 & 4k & +Ctrl & $7.7$ & $3.5{\times}10^{4}$ & $0.086$ & $0.190$ & $39.4$ \\
G1 & 4k & +Aux & $7.2$ & $1.6{\times}10^{4}$ & $0.089$ & $0.191$ & $42.4$ \\
\midrule
G1 & 140k & Baseline & $8.0$ & $6.8{\times}10^{4}$ & $0.781$ & $0.130$ & $62.6$ \\
G1 & 140k & +Cov & $8.0$ & $5.1{\times}10^{4}$ & $0.406$ & $0.131$ & $64.1$ \\
G1 & 140k & +Ctrl & $7.7$ & $3.3{\times}10^{4}$ & $0.043$ & $0.132$ & $64.2$ \\
G1 & 140k & +Aux & $7.1$ & $1.6{\times}10^{4}$ & $0.042$ & $0.132$ & $64.5$ \\
\midrule
Go2 & 4k & Baseline & $38.6$ & $1.7{\times}10^{4}$ & $0.609$ & $0.259$ & $36.7$ \\
Go2 & 4k & +Cov & $37.3$ & $7.4{\times}10^{3}$ & $0.601$ & $0.260$ & $41.7$ \\
Go2 & 4k & +Ctrl & $31.6$ & $1.1{\times}10^{4}$ & $0.093$ & $0.250$ & $58.5$ \\
Go2 & 4k & +Aux & $32.7$ & $4.6{\times}10^{3}$ & $0.093$ & $0.252$ & $61.6$ \\
\midrule
Go2 & 140k & Baseline & $43.4$ & $2.4{\times}10^{4}$ & $0.518$ & $0.204$ & $63.4$ \\
Go2 & 140k & +Cov & $28.8$ & $1.1{\times}10^{4}$ & $0.442$ & $0.203$ & $65.3$ \\
Go2 & 140k & +Ctrl & $31.1$ & $2.2{\times}10^{4}$ & $0.053$ & $0.209$ & $82.4$ \\
Go2 & 140k & +Aux & $26.7$ & $4.8{\times}10^{3}$ & $0.053$ & $0.208$ & $100.6$ \\
\midrule
Kinova & 4k & Baseline & $7.9$ & $9.9{\times}10^{7}$ & $0.826$ & $1.08{\times}10^{-3}$ & -- \\
Kinova & 4k & +Cov & $7.8$ & $3.2{\times}10^{7}$ & $0.741$ & $9.86{\times}10^{-4}$ & -- \\
Kinova & 4k & +Ctrl & $8.2$ & $2.5{\times}10^{7}$ & $0.207$ & $1.01{\times}10^{-3}$ & -- \\
Kinova & 4k & +Aux & $8.4$ & $2.8{\times}10^{7}$ & $0.224$ & $9.18{\times}10^{-4}$ & -- \\
\midrule
Kinova & 140k & Baseline & $8.4$ & $7.7{\times}10^{9}$ & $0.843$ & $5.49{\times}10^{-4}$ & -- \\
Kinova & 140k & +Cov & $8.0$ & $3.0{\times}10^{9}$ & $0.757$ & $5.37{\times}10^{-4}$ & -- \\
Kinova & 140k & +Ctrl & $7.3$ & $6.7{\times}10^{8}$ & $0.106$ & $5.37{\times}10^{-4}$ & -- \\
Kinova & 140k & +Aux & $8.6$ & $3.5{\times}10^{9}$ & $0.127$ & $5.10{\times}10^{-4}$ & -- \\
\bottomrule
\end{tabular}
\end{table*}

\subsection{Selection of Closed-Loop Settings}
\label{sec:ap_aux_tuning}

\paragraph{Validation-based selection.} For the G1 and Go2 closed-loop studies, we select the training length, the auxiliary-loss weights, and the MPC input weight $r$ in $R=r\mathbf{I}$ on validation reference trajectories (1000--1499) that are disjoint from the test trajectories (0--999). The baseline and the regularized models use the same candidate training lengths ($20$k or $60$k steps) and input weights ($r\in\{0,10^{-3}\}$); the regularized models additionally select their loss weights ($w_\text{cov}\in\{0.1,1\}$, $w_\text{ctrl}\in\{0.01,0.1,1\}$), so their search is larger. Selection uses $m=140$k, $n_\text{mult}\in\{1,4,16\}$, and two seeds. With the selected settings, we train all four variants on the full grid ($m\in\{1,4,16,64,140\}$k, $n_\text{mult}\in\{1,2,4,8,16\}$, five seeds) and evaluate them on the test trajectories.

\paragraph{Selected settings.} On G1, both the baseline and the regularized models are best with $60$k training steps, and +Aux uses $w_\text{cov}=1$, $w_\text{ctrl}=0.1$, and $r=10^{-3}$. With $60$k steps, every +Aux setting has higher validation survival than the best baseline ($62.3$--$64.7$ versus $60.8$ steps), so the gain does not depend on a narrow choice of weights. With the default settings ($w=1$, $20$k steps), the gain on the test trajectories is smaller ($9.6\%$ and $6.5\%$ in the two sweeps, versus $10.0\%$ and $8.0\%$ in \cref{tab:control_summary}). On Go2, the selected settings are $20$k steps, $w_\text{cov}=w_\text{ctrl}=1$, and $r=10^{-3}$, which coincide with the default settings. Here the inverse-control weight matters most: every +Aux setting with $w_\text{ctrl}=1$ survives longer on the validation trajectories than the best baseline ($82.8$--$96.3$ versus $67.5$ steps), and every setting with $w_\text{ctrl}=0.01$ survives shorter ($38.2$--$45.6$ steps). The +Cov and +Ctrl models use the selected weight of their own loss. The complete search results are in the code repository.

\paragraph{Test results.} At $m=140$k, averaged over latent sizes, G1 survives $61.4\pm0.8$ steps with the baseline, $62.6\pm1.0$ with +Cov, $63.5\pm0.8$ with +Ctrl, and $64.9\pm0.7$ with +Aux (mean $\pm$ standard deviation over five seeds); Go2 survives $66.4\pm4.7$, $71.6\pm7.6$, $89.3\pm4.7$, and $98.7\pm8.6$ steps. +Aux is better than the baseline in all five seeds on both robots, by $1.8$--$4.9$ steps on G1 and $11.9$--$38.9$ steps on Go2 (paired $t$-test, $p=0.003$), including the three seeds that the selection never used. Because the baseline and +Aux use the same selected training length, the gain does not come from longer training.

\subsection{Franka Closed-Loop Tracking}
\label{sec:ap_franka_control}

The Franka arm tracks a figure-8 end-effector reference with the MPC of \cref{sec:koopman_mpc_form}. Franka has no separate set of reference trajectories, so all settings are selected on a star-shaped trajectory and evaluated on the figure-8 trajectory.

\paragraph{Selection.} We select the +Aux loss setting and the MPC setting jointly on the star trajectory. The +Aux models ($m=140$k, $n_\text{mult}\in\{1,4,16\}$, two seeds) are trained with $w_\text{cov}\in\{0.01,0.1,1\}$, $w_\text{ctrl}\in\{0.1,1,3\}$, and $120$k or $240$k steps, and each is evaluated with 80 MPC settings: horizon $H\in\{5,10,15,25\}$, input weight $r\in\{0,10^{-3},10^{-2},10^{-1},1\}$, and four choices of the joint-velocity reference and weight. The selected setting is $w_\text{cov}=1$, $w_\text{ctrl}=0.1$, $120$k steps, and $H=15$, $r=10^{-2}$ without a cost on joint velocities. This MPC setting is applied unchanged to every method, including EDMD. The baseline, +Cov, and +Ctrl models use $120$k steps and $w=1$. Removing the velocity cost matters for every method: the reference velocity is zero, and with joint-velocity commands this cost acts as damping that makes the arm lag the moving reference. For the baseline at $m=140$k and $r=0$, removing it lowers the figure-8 error from $0.334$ to $0.279$; the selected setting gives $0.233$.

\begin{figure}[htbp]
    \centering
    \includegraphics[width=1.0\linewidth]{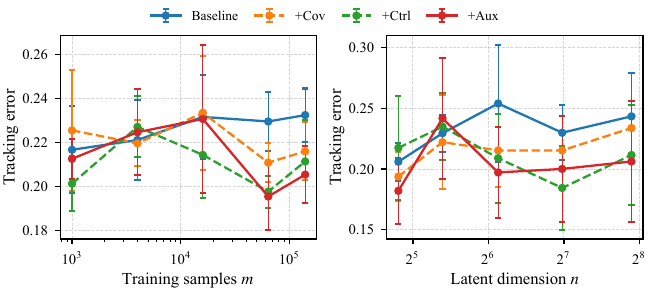}
    \caption{Franka figure-8 tracking error versus training samples $m$ (left; mean over the five latent sizes) and latent dimension $n$ (right; $m=140$k), with the MPC setting selected on the star trajectory. Error bars: one standard deviation across seeds (left) and models (right).}
    \label{fig:control_franka}
\end{figure}

\paragraph{Test.} \Cref{fig:control_franka} shows the figure-8 results. At $m=140$k, averaged over latent sizes, the tracking error is $0.233\pm0.012$ for the baseline, $0.216\pm0.013$ for +Cov, $0.211\pm0.019$ for +Ctrl, and $0.205\pm0.013$ for +Aux (mean $\pm$ standard deviation over five seeds). +Aux is better than the baseline in all five seeds (paired $t$-test, $p=0.011$). Averaged over the points of each sweep, +Aux lowers the error by $11.4\%$ in the latent-dimension sweep and $5.4\%$ in the data sweep, where +Ctrl gives $7.0\%$. With the same MPC, EDMD reaches $0.264$ and NNDM with random shooting $0.254$ (\cref{tab:baseline_control}), so every Koopman variant tracks more accurately than both baselines. As a check, when each Koopman variant instead uses the MPC setting selected on its own star-trajectory runs, the errors at $m=140$k are $0.220$ (baseline), $0.204$ (+Cov), $0.204$ (+Ctrl), and $0.205$ (+Aux): the regularized variants remain about $7\%$ better than the baseline.

\paragraph{Why the inverse-control loss matters here.} The MPC re-encodes the measured state at every step, whereas $\mathcal{L}_\text{pred}$ encodes only the first state of each training segment and rolls the latent state forward. Each Franka segment starts with a large joint-velocity transient, so states inside a segment are never encoded during training with $\mathcal{L}_\text{pred}$ alone. On held-out transitions at $m=140$k, the median one-step error of predictions made from re-encoded measured states (normalized states) is $3.9\times10^{3}$ for the baseline and $3.7\times10^{3}$ for +Cov, but $3.4$ for +Ctrl and $5.6$ for +Aux: $\mathcal{L}_\text{ctrl}$ encodes measured states at every step and thereby makes the model consistent on the states that the controller sees.

\end{document}